\documentclass[twoside,11pt]{article}

\usepackage{blindtext}

\usepackage[utf8]{inputenc}
\usepackage[T1]{fontenc}
\usepackage{lmodern}
\usepackage{hyperref}
\usepackage{url}            % simple URL typesetting
\usepackage{booktabs}       % professional-quality tables
\usepackage{amsfonts}       % blackboard math symbols
\usepackage[numbers]{natbib}
\usepackage{nicefrac}       % compact symbols for 1/2, etc.
\usepackage{algorithm,algpseudocode}
\usepackage{amsfonts}
\usepackage{multicol}
\usepackage{amsmath}
\usepackage{amsthm}
\usepackage{amssymb}
\usepackage[title]{appendix}
\usepackage{bm}
\usepackage{courier}
\usepackage[usenames,dvipsnames]{color}
\usepackage{enumitem}
\usepackage{graphicx}
\usepackage[margin=0.9in]{geometry}
\usepackage{url}
\usepackage{multirow}
\usepackage[dvipsnames]{xcolor}
\usepackage{epsfig}
\usepackage{graphicx}
\usepackage{subcaption}
\usepackage{wrapfig}   % <-- defines wrapfigure
\usepackage{caption}   % (optional) nicer caption spacing
\hypersetup{
	colorlinks,
	linkcolor={red!80!black},
	citecolor={blue!50!black},
	urlcolor={blue!80!black}
}

\usepackage{subcaption}
\usepackage{graphicx}
\usepackage{caption}
\title{SOCKET: SOft Collision Kernel EsTimator for Sparse Attention}

\newcommand{\samethanks}[1][\value{footnote}]{\footnotemark[#1]}

\date{}
\author{
  Sahil Joshi\thanks{\scriptsize Department of Computer Science, Rice University, TX, USA. \texttt{\{sj157,ac508,wb20,ask20,es100,el72,as143\}@rice.edu}}%
  \and Agniva Chowdhury\samethanks[1]
  \and Wyatt Bellinger\samethanks[1]
  \and Amar Kanakamedala\samethanks[1]
  \and Ekam Singh\samethanks[1]
  \and Hoang Anh Duy Le\samethanks[1]
  \and Aditya Desai\thanks{\scriptsize Department of Electrical Engineering and Computer Sciences, UC Berkeley, CA, USA. \texttt{apdesai@berkeley.edu}}
  \and Anshumali Shrivastava\samethanks[1]
}

\usepackage{bibentry}

\input{./Definitions}

\begin{document}

\maketitle

\begin{abstract}
\noindent
Exploiting sparsity is key to efficient long-context inference, as attention dominates the cost of autoregressive decoding. Sparse attention reduces this cost by restricting computation to a subset of tokens, but its effectiveness hinges on fast and accurate token scoring and selection at inference time. Data-agnostic approaches offer an attractive way to perform this selection, but often incur substantial memory overhead to maintain high recall. We revisit Locality-Sensitive Hashing (LSH) and introduce SOCKET, a SOft Collision Kernel EsTimator that replaces hard bucket matches with probabilistic, similarity-aware aggregation. Traditional LSH relies on binary collision signals, providing limited information for ranking tokens and necessitating many hash tables for accurate retrieval. In contrast, soft LSH accumulates graded collision evidence across hash tables, closely preserving the true top-$k$ ordering with significantly less memory. This reframes LSH from a candidate-generation mechanism into a principled scoring kernel for sparse attention. Building on this insight, SOCKET enables efficient token selection without ad hoc voting and matches or outperforms existing sparse attention methods across multiple long-context benchmarks and diverse language models. With a custom set of CUDA/Triton kernels for scoring, selection, and attention, SOCKET achieves up to approximately $1.5\times$ higher throughput than FlashAttention. Code is open-sourced at \url{https://github.com/amarka8/SOCKET}.
\end{abstract}

\section{Introduction}
\label{sec:intro}
Large language models (LLMs), such as GPT~\citep{brown2020language} and Llama~\citep{touvron2023llama}, have achieved remarkable performance in next-token prediction, enabling a wide range of applications in language understanding and generation~\citep{openai2023gpt4, chowdhery2022palm, hoffmann2022training,  raffel2020exploring}. To support more powerful in-context learning and increasingly complex applications~\citep{dong2022survey, press2022train, ratner2023parallel, wei2022chain}, the maximum input length supported during inference has grown rapidly--from 2K–4K tokens~\citep{devlin2019bert, vaswani2017attention} to 32K~\citep{beltagy2020longformer, chowdhery2022palm}, 128K~\citep{openai2023gpt4, touvron2023llama}, and even millions of tokens~\citep{zhang2023ringattention}. Beyond language, LLMs have also been extended to multimodal domains such as vision, video, code, databases, and scientific discovery~\citep{jumper2021highly, li2023blip, li2023starcoder, radford2021learning,  wu2023video, zhou2024dbgpt}.
\\
\\
\noindent
LLM inference consists of an initial prefilling phase followed by iterative decoding steps that generate one token at a time. During prefilling, the model processes the full input sequence and computes key-value (KV) representations for all tokens. During decoding, only the most recently generated token is processed, and its KV representations are appended to the cache, avoiding redundant computation. However, each decoding step requires the new token to attend to all previously cached tokens, making attention computation increasingly memory-bound as the context grows. When the KV cache exceeds GPU memory capacity and is partially offloaded to CPU memory, additional CPU-GPU transfers further exacerbate this bottleneck. These effects reveal a fundamental scalability limitation of dense attention for long-context inference. Sparse attention addresses this limitation by restricting attention to a selected subset of tokens, reducing memory movement while approximating dense attention behavior.

The attention output for a query $\mathbf{q}$ under scaled dot-product attention (SDPA) is given by
\begin{equation}
\mathbf{y}(\mathbf{q}) \;=\; \sum_{i=1}^{n} a_i \, \mathbf{v}_i,
\label{eq:sdpa}
\end{equation}
where
$a_i = \nicefrac{\exp(\mathbf{k}_i^\top \mathbf{q})}{\sum_{j=1}^{n} \exp(\mathbf{k}_j^\top \mathbf{q})}$ and, $\mathbf{k}_i,\mathbf{v}_i\in\R{d}$  denote the key and value vectors associated with token $i$, and $a_i$ represents its attention weight\footnote{For simplicity, we avoid using the $\nicefrac{1}{\sqrt{d}}$ factor in the definition.}. The central challenge in approximating dense attention lies in identifying the tokens that contribute most to this sum. Prior work has shown that, in hindsight, the dominant contributors are those with large values of $a_i \lVert \mathbf{v}_i \rVert_2$~\citep{desai2025hashattention}. This observation motivates selecting tokens with the largest query--key inner products $\mathbf{k}_i^\top \mathbf{q}$, commonly referred to as \emph{top-$k$} selection. Related extensions, such as top-$p$, further adapt this strategy by dynamically allocating token budgets across attention heads. Consequently, much of the sparse attention literature has focused on efficiently approximating top-$k$ selection~\citep{desai2025hashattention, li2024b, hooper2024fast, pmlr-v235-tang24l, yang2024double, pqcache2025} or top-$p$ selection~\citep{zhu2025topp}.
\\

% ------------------ TABLE ------------------
\begin{figure*}[t]
\centering

% =========================================================
% LEFT: TABLE
% =========================================================
\begin{minipage}[t]{0.56\textwidth}
\vspace{0pt}
\centering
\scriptsize
\setlength{\tabcolsep}{2.8pt}
\renewcommand{\arraystretch}{0.92}

\begin{tabular}{@{}lrrrrrrr@{}}
\toprule
Method & nm2 & nm3 & vt & fwe & qa1 & qa2 & Avg \\
\midrule

Dense
& 96 & 100 & 95.2 & 93.3 & 82 & 52 & 86.4 \\

\midrule
\multicolumn{8}{c}{\textit{5$\times$ sparsity}} \\
\midrule

Quest {\scriptsize (ps=32)}
& 91 & 92 & 97.2 & 90.3 & 82 & 52 & 84.1 \\

DS {\scriptsize ($l$=8, ch=32)}
& 93 & 100 & 97.0 & 92.0 & 81 & 52 & 85.8 \\

HashAttn {\scriptsize ($b$=128)}
& 92 & 99 & 90.4 & 93.3 & 82 & 50 & 84.5 \\

PQCache {\scriptsize (pq=8, gr=4)}
& 97 & 99 & 95.2 & 95.4 & 81 & 51 & 86.4 \\

MagicPig {\scriptsize (P=7, L=100)}
& 49 & 23 & 93.3 & 92.4 & 63 & 47 & 61.3 \\

SOCKET {\scriptsize (P=10, L=60)}
& \textbf{97} & \textbf{100} & 95.2 & \textbf{96}
& 81 & \textbf{53} & \textbf{87.0} \\

\midrule
\multicolumn{8}{c}{\textit{10$\times$ sparsity}} \\
\midrule

Quest {\scriptsize (ps=32)}
& 86 & 45 & 96 & 90 & 80 & 53 & 75.0 \\

DS {\scriptsize ($l$=8, ch=32)}
& 76 & 95 & 94 & 90.6 & 72 & 50 & 79.6 \\

HashAttn {\scriptsize ($b$=128)}
& 87 & 89 & 85.6 & 92.6 & 79 & 52 & 80.9 \\

PQCache {\scriptsize (pq=8, gr=4)}
& 98 & 98 & 94.0 & 92.3 & 81 & 51 & 85.7 \\

MagicPig {\scriptsize (P=8, L=120)}
& 10.3 & 0 & 93.0 & 92.0 & 47 & 42 & 47.4 \\

SOCKET {\scriptsize (P=10, L=60)}
& 96 & \textbf{99.6} & \textbf{94.2} & \textbf{96.4}
& 80 & \textbf{52} & \textbf{86.4} \\

\midrule
\multicolumn{8}{c}{\textit{20$\times$ sparsity}} \\
\midrule

Quest {\scriptsize (ps=16)}
& 89 & 69 & 95 & 90 & 81 & 53 & 79.5 \\

DS {\scriptsize ($l$=16, ch=32)}
& 49 & 82 & 90 & 91 & 68 & 51 & 71.8 \\

HashAttn {\scriptsize ($b$=128)}
& 73 & 56 & 82 & 90 & 76 & 49 & 71.0 \\

PQCache {\scriptsize (pq=8, gr=4)}
& 93 & 94 & 89.2 & 81.0 & 81 & 50 & 81.4 \\

MagicPig {\scriptsize (P=8, L=75)}
& 1 & 0 & 75.4 & 84.8 & 31 & 33 & 37.5 \\

SOCKET {\scriptsize (P=10, L=60)}
& \textbf{93} & \textbf{94} & \textbf{91.4} & \textbf{94.7}
& 80 & \textbf{51.3} & \textbf{84.1} \\

\midrule
\multicolumn{8}{c}{\textit{50$\times$ sparsity}} \\
\midrule

Quest {\scriptsize (ps=16)}
& 52 & 7 & 91 & 88 & 65 & 51 & 59.0 \\

DS {\scriptsize ($l$=16, ch=32)}
& 28 & 8 & 84 & 76 & 48 & 44 & 48.0 \\

HashAttn {\scriptsize ($b$=128)}
& 40 & 7 & 74 & 86 & 64 & 45 & 52.7 \\

PQCache {\scriptsize (pq=8, gr=4)}
& 72 & 79 & 89.0 & 81.0
& 77 & 50 & 74.7 \\

MagicPig {\scriptsize (P=10, L=150)}
& 0.3 & 0 & 70.6 & 82.7 & 21.6 & 28 & 33.9 \\

SOCKET {\scriptsize (P=10, L=60)}
& \textbf{80} & 73 & 83.8 & \textbf{90.3}
& \textbf{77} & \textbf{50} & \textbf{75.7} \\

\bottomrule
\end{tabular}

\captionof{table}{
Performance on RULER-HARD-32K using Llama-3.1-8B-Instruct.
}
\label{tab:llama-ruler}

\end{minipage}
\hfill
% =========================================================
% RIGHT: FIGURE
% =========================================================
\begin{minipage}[t]{0.41\textwidth}
\vspace{0pt}
\centering

\includegraphics[width=\linewidth]{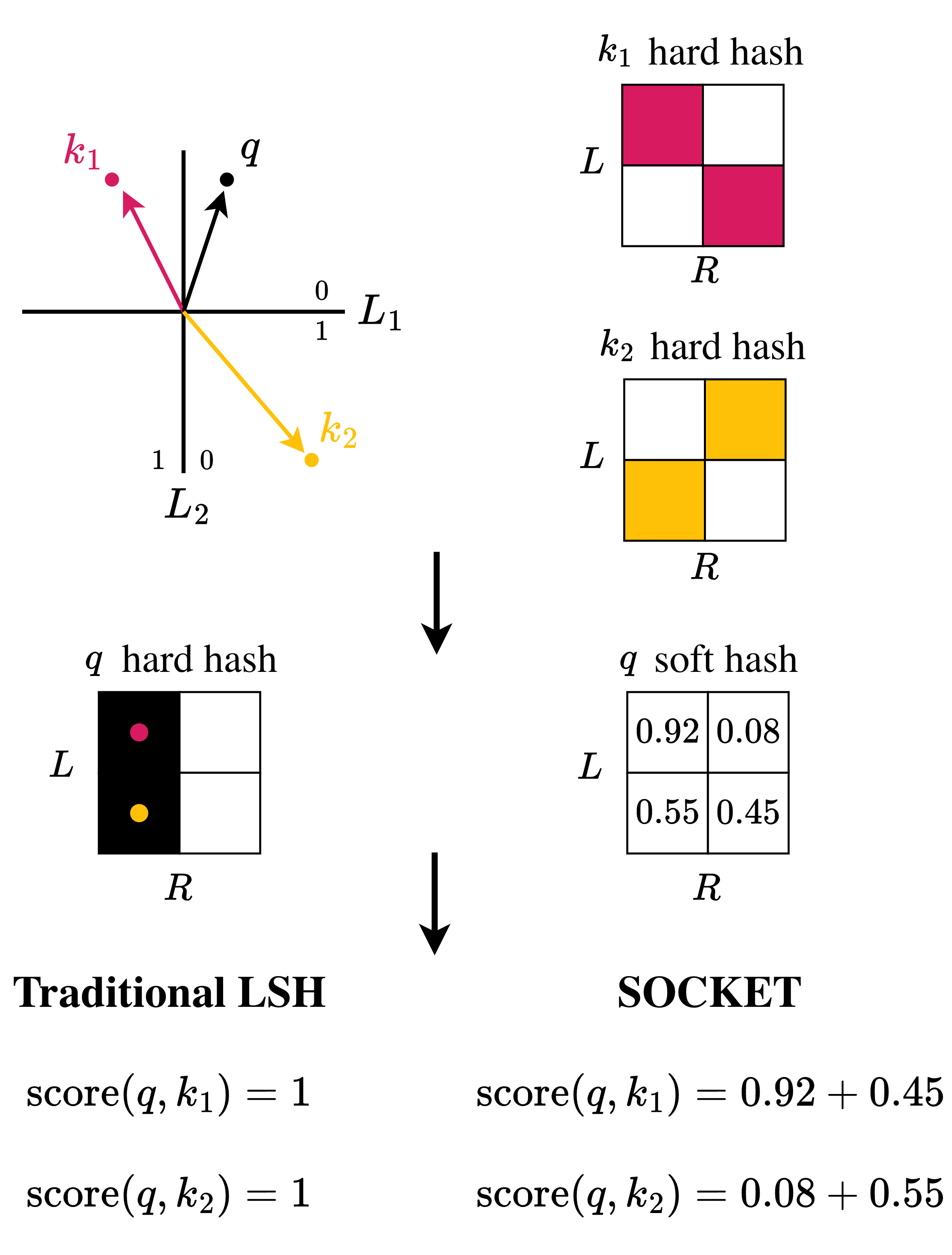}

\captionof{figure}{
Traditional LSH assigns binary scores based on hash collisions,
while soft LSH produces continuous scores via probabilistic bucket
assignments. Soft LSH therefore induces a stable ranking:
if $k_1$ is closer to $q$ than $k_2$, then
$\mathrm{score}(q,k_1) > \mathrm{score}(q,k_2)$.
}
\label{fig:soft-vs-hard}

\end{minipage}
\end{figure*}

% =========================
% Main table: means only
% =========================
\begin{table*}[t]
\centering
\footnotesize
\setlength{\tabcolsep}{4.5pt}
\renewcommand{\arraystretch}{0.95}
\caption{Performance on RULER-HARD-32K using Ministral-3-8B-Instruct-2512.}
\label{tab:ministral-ruler}

\begin{tabular}{@{}lcrrrrrrrr@{}}
\toprule
Method & Spr & nm2 & nm3 & vt & fwe & qa1 & qa2 & Avg \\
\midrule

Dense & --
& 99 & 100 & 100 & 94.3 & 79 & 63 & 89.2 \\

\midrule

Quest {\scriptsize (ps=32)} & $5\times$
& 99 & 89 & 100 & 96.3 & 79 & 60 & 87.2 \\

DS {\scriptsize ($l$=8, ch=32)} & $5\times$
& 99 & 99 & 100 & 93 & 76 & 62 & 88.2 \\

PQCache {\scriptsize (pq=8, gr=4)} & $5\times$
& 99 & 100 & 100 & 94.3 & 79 & 57 & 88.2 \\

MagicPig {\scriptsize (P=7, L=100)} & $5\times$
& 85.3 & 60.3 & 100 & 94.3 & 77.3 & 59.6 & 79.5 \\

SOCKET {\scriptsize (P=10, L=60)} & $5\times$
& \textbf{99} & \textbf{100} & \textbf{100} & \textbf{97.6} & \textbf{80} & \textbf{63} & \textbf{89.9} \\

\midrule

Quest {\scriptsize (ps=32)} & $10\times$
& 85 & 39 & 100 & 97 & 81 & 60 & 77.0 \\

DS {\scriptsize ($l$=8, ch=32)}& $10\times$
& 97 & 81 & 100 & 92.3 & 76 & 56 & 83.7 \\

PQCache {\scriptsize (pq=8, gr=4)}& $10\times$
& 99 & 99 & 100 & 92.6 & 79 & 63 & 88.8 \\

MagicPig {\scriptsize (P=8, L=120)} & $10\times$
& 47.6 & 11 & 99.8 & 93.4 & 69.6 & 54 & 62.6 \\

SOCKET {\scriptsize (P=10, L=60)} & $10\times$
& \textbf{99} & 98 & \textbf{100} & \textbf{98$^{\star}$} & \textbf{81} & 61 & \textbf{89.5} \\

\midrule

Quest {\scriptsize (ps=16)} & $20\times$
& 95 & 65 & 99.2 & 97.6 & 78 & 58 & 82.1 \\

DS {\scriptsize ($l$=16, ch=32)}& $20\times$
& 90 & 39 & 98 & 90 & 73 & 53 & 73.8 \\

PQCache {\scriptsize (pq=8, gr=4)}& $20\times$
& 99 & 97 & 100 & 91.3 & 81 & 60 & 88.1 \\

MagicPig {\scriptsize (P=8, L=75)} & $20\times$
& 2.3 & 0 & 98.3 & 90.5 & 44.6 & 38.6 & 45.7 \\

SOCKET {\scriptsize (P=10, L=60)}& $20\times$
& \textbf{99} & 93 & \textbf{100} & \textbf{98$^{\star}$} & 80 & \textbf{60} & \textbf{88.3} \\

\midrule

Quest {\scriptsize (ps=8)} & $50\times$
& 91 & 64 & 100 & 95.6 & 79 & 57 & 81.1 \\

DS {\scriptsize ($l$=16, ch=32)} & $50\times$
& 63 & 4 & 88.4 & 85.6 & 60 & 49 & 58.3 \\

PQCache {\scriptsize (pq=8, gr=4)}& $50\times$
& 95 & 82.6 & 100 & 89.8 & 81.6 & 56.6 & 84.3 \\

MagicPig {\scriptsize (P=10, L=150)} & $50\times$
& 0.3 & 0 & 93.3 & 90.5 & 34.3 & 31 & 41.6 \\

SOCKET {\scriptsize (P=10, L=60)}& $50\times$
& \textbf{98} & 82 & \textbf{100} & 91.1
& 79.3 & 57.6 & \textbf{84.7} \\

\bottomrule
\end{tabular}
\end{table*}

\noindent
Locality-Sensitive Hashing (LSH)~\cite{IndykMotwani1998LSH}, in particular, is a widely used randomized technique for identifying similar keys in high-dimensional spaces. In its traditional form, LSH applies random projections followed by a sign function, producing binary collision indicators between a query and a key. However, this \emph{hard} formulation is poorly suited for ranking stability of important keys: binary collisions cannot express partial similarity, leading to coarse and noisy rankings (see fig.~\ref{fig:ranking_metrics}). Soft LSH resolves this limitation by aggregating collision evidence across multiple hash tables to produce lightweight, similarity-aware scores for every key. Recently, randomized data structures have emerged as a promising approach for approximating attention. RACE-based sketches~\citep{pmlr-v119-coleman20a, ColemanS20-RACE-KDE} have been used to approximate softmax attention in linear time~\citep{joshi2026race}, while LSH-based methods have been explored to approximate attention and retrieval~\citep{gionis1999similarity, kitaev2020reformer, liu2011hashing}. These works highlight the potential of randomized, data-agnostic primitives to reduce attention cost without relying on expensive data-dependent preprocessing.

\noindent
Unlike most existing sparse attention methods, SOCKET is \emph{data-agnostic}. This design choice yields three key advantages. First, it enables substantially faster time-to-first-token (TTFT) compared to data-dependent clustering approaches such as $k$-means, as shown in fig.~\ref{fig:pq_gpu}. Second, it enables deployment without retraining or calibration, making it robust to distribution shifts during inference. Finally, it allows for interpretable theoretical guarantees with respect to a kernel that closely mimics the softmax kernel, as formalized in Theorem~\ref{thm:main}.

\begin{table}[ht]
\centering
\setlength{\tabcolsep}{4pt}
\caption{Computational cost and memory overhead during retrieval for SOCKET and traditional LSH. All results are benchmarked using our custom CUDA scoring kernel. Avg Score represents the average retrieval accuracy for RULER-HARD-32K dataset using Llama-3.1-8B-Instruct.}
\label{tab:method_comparison}
\vspace{3pt}
\begin{tabular}{lcccccc}
\hline
\textbf{Method} & \textbf{(P, L)} & \textbf{Memory (GB)} & \textbf{Overhead} & \textbf{Time (ms)} & \textbf{Overhead} & \textbf{Avg Score} \\
\hline
SOCKET & (10, 60)  & 1.048 & 1.00$\times$ & 0.387 & 1.00$\times$ & \textbf{85.08} \\
LSH    & (10, 60)  & 1.049 & 1.00$\times$ & 0.387 & 1.00$\times$ & 10.00 \\
% LSH    & (2, 300)  & 2.953 & 2.81$\times$ & 1.027 & 2.65$\times$ & 84 \\
LSH    & (2, 300)  & 2.953 & 2.81$\times$ & 1.027 & 2.65$\times$ & 84 \\
% LSH    & (2, 350)  & 3.353 & 3.19$\times$ & 1.209 & 3.12$\times$ & 84.04 \\
LSH    & (2, 400)  & 3.753 & 3.57$\times$ & 1.328 & 3.43$\times$ & 83.76 \\
% LSH    & (2, 450)  & 4.153 & 3.96$\times$ & 1.480 & 3.82$\times$ & 83.36 \\
LSH    & (2, 500)  & 4.554 & 4.34$\times$ & 1.629 & 4.20$\times$ & 84.80 \\
\hline
\end{tabular}
\end{table}
\noindent
\textbf{Key Idea:} Our preliminary experiments with a \emph{hard} LSH–based scorer show that it fails to rank candidate keys according to their true  under the same memory budget, as illustrated in fig.~\ref{fig:ranking_metrics}. This observation motivates a re-examination of the role of LSH in sparse attention and raises a fundamental question: \emph{Can we augment traditional LSH to more effectively rank candidates by their relevance?}
\\
\\
\noindent
In \textsc{SOCKET}, each key is assigned a continuous score based on soft collision probabilities across multiple hash tables, weighted by the corresponding value vector norms. This turns LSH from a binary filter into a similarity-aware scoring kernel that enables stable key ranking. Crucially, these scores allow efficient key selection without accessing full key or value vectors. From a systems perspective, exact top-$k$ selection requires reading full key vectors (e.g., 128 bfloat values per token), whereas \textsc{SOCKET} uses compact hash representations, retrieving only precomputed bucket indices (e.g., $\sim$600 bits per token) to compute query-dependent soft counts, along with a single integer (value norm) per key. This substantially reduces memory traffic during decoding while preserving ranking fidelity. 
% This substantially reduces memory traffic during decoding while preserving ranking fidelity. 
We present ablations of both \textsc{SOCKET} and a \textit{hard} LSH estimator in Sections~\ref{app:ablation_socket} and~\ref{app:ablations_lsh}. The results show that traditional LSH requires many more hash tables to match \textsc{SOCKET}'s retrieval accuracy, incurring higher memory overhead and latency (Table~\ref{tab:method_comparison}). Overall, \textsc{SOCKET} offers a better accuracy-efficiency trade-off for data-agnostic sparse attention. In addition to these findings and extensive empirical validation, we make the following contributions:\\
\textbf{I. Soft LSH as a \emph{stable ranker}:} We propose a \textit{data-agnostic soft} LSH mechanism that serves as a substantially more effective ranking function than LSH, making it well suited for sparse attention. \\
\textbf{II. Efficiency with SOCKET:} We design a custom CUDA kernel for scoring of keys and show a $1.5\times$ throughput speedup over FlashAttention during decoding using GPT-FAST\footnote{https://github.com/meta-pytorch/gpt-fast}. \\
\textbf{III. Theoretical Insights:} We provide theoretical guarantees showing that soft LSH closely approximates a softmax-like attention kernel. Additionally, we theoretically show why soft LSH yields more accurate and stable token rankings than hard LSH.

\section{Related Work}
Given the importance of long-context inference, extensive prior work accelerates attention by approximating top-$k$ selection and restricting computation to a subset of tokens per query. One line of work uses dimensionality reduction: \textsc{Double Sparsity}~\citep{yang2024double} reduces computation along the feature dimension by selecting channels via offline calibration of channel norms, while \textsc{Quest}~\citep{pmlr-v235-tang24l} reduces computation along the token dimension by selecting tokens at the page level. Although effective, these approaches rely on data-dependent heuristics and lack theoretical guarantees. In Section~\ref{sec:exp}, we show that SOCKET matches or outperforms these methods on diverse set of models while providing a principled, data-agnostic alternative. A second line of work casts top-$k$ attention as retrieval. \textsc{RetrievalAttention}~\citep{liu2024retrievalattention} uses graph-based nearest neighbor search but offloads selection to the CPU due to irregular computation, introducing latency. \textsc{PQCache}~\citep{pqcache2025} integrates product quantization into KV-cache management for approximate retrieval, also relying on a CPU--GPU pipeline. In contrast, we show that data-agnostic random projections significantly accelerate TTFT while maintaining or improving accuracy (fig.~\ref{fig:pq_gpu}). SOCKET further macthes or outperforms \textsc{PQCache} on a diverse set of models, while offering a simpler design (see Section~\ref{sec:exp}).
\\
\\
\noindent
Furthermore, several approaches estimate attention outputs using hashing-based data structures. \textsc{Learning-to-Hash Attention}~\citep{sun2022sparse} trains hash functions to map keys into balanced hash tables, enabling sparse retrieval during attention computation. \textsc{HashAttention}~\citep{desai2025hashattention} encodes queries and keys into Hamming space using learned mappings to capture semantic similarity. \textsc{MagicPig}~\citep{chen2024magicpig} observes that top-$k$ attention can be biased when attention scores are relatively uniform, and addresses this issue with an importance-sampling estimator built on LSH that approximates the sampling distribution. Concretely, it uses traditional LSH to sample candidate keys and applies an importance-sampling correction to obtain an unbiased estimator of softmax attention. In contrast, \textsc{SOCKET} employs \textit{soft} LSH to compute soft collision scores, deterministically selects the top-$k$ keys, and subsequently performs exact attention over the retrieved subset. Therefore, while MagicPig is fundamentally a sampling-based estimator, \textsc{SOCKET} is a retrieval-based approach centered on accurate top-$k$ selection. While these methods highlight the potential of hashing for reducing attention cost, they primarily rely on hard or learned hash functions and do not directly address the problem of accurate key ranking. In Section~\ref{sec:exp}, we show that SOCKET's soft LSH scoring improves performance on aggregation, summarization, and retrieval tasks using significantly lesser memory than MagicPig.
\\
\\
\noindent
Several other works address efficient long-context inference, including \textsc{vAttention}~\citep{desai2025vattention}, \textsc{Squeeze Attention}~\citep{hooper2024fast}, \textsc{Loki}~\citep{loki2024lowrank}, \textsc{InfLLM}~\citep{xiao2024infllm}, \textsc{StreamLLM}~\citep{xiao2023streamllm}, \textsc{Adamas}~\citep{zhang2024adamas}, and \textsc{H$_2$O}~\citep{zhang2023h2o}, which explore complementary system-level and algorithm-level strategies to reduce attention overhead. We do not aim to provide a comprehensive empirical comparison with these methods. Instead, we focus on a principled, data-agnostic sparsification approach based on locality-sensitive hashing, which serves as a similarity-aware scoring mechanism rather than a binary filtering heuristic. Accordingly, we do not discuss these methods further.

\section{Background}
\subsection{Locality-Sensitive Hashing (LSH)}
An LSH family $\mathcal H$ for a similarity $\mathrm{Sim}$ makes near pairs collide more often than far pairs. Formally, $\mathcal H$ is $(S_0,cS_0,p_1,p_2)$-sensitive if for all $x,y\in\mathbb{R}^D$,
\[
\begin{cases}
\mathrm{Sim}(x,y)\ge S_0 \;\Rightarrow\; \Pr_{h\sim\mathcal H}[h(x)=h(y)]\ge p_1,\\[2pt]
\mathrm{Sim}(x,y)\le cS_0 \;\Rightarrow\; \Pr_{h\sim\mathcal H}[h(x)=h(y)]\le p_2,
\end{cases}
\]
% \[
% \text{where } p_1 > p_2 \text{ and } c < 1.
% \]
where $p_1 > p_2$ and $c < 1$. Such families enable sublinear-time approximate nearest-neighbor data structures. A convenient sufficient condition, satisfied by SimHash and WTA \citep{Charikar2002, ChenShrivastava2018, Yagnik2011}, is that the collision probability is a monotone function of similarity,
$\Pr_{h\sim\mathcal H}[h(x)=h(y)] = f(\mathrm{Sim}(x,y))$ with $f$ increasing.

\subsection{Sparse Attention}
For clarity, we restrict our exposition to the case of batch size one with a single query vector $\qb\in\R{d}$. Consider keys $\kb_1,\dots,\kb_N\in \mathbb{R}^d$ and corresponding values $\vb_1,\dots,\vb_N\in \mathbb{R}^d$. We modify~\eqref{eq:sdpa} to define sparse attention. Let $S$ denote the sequence of $k$ important token indices selected by a method. The sparse attention computation based on this index set is given by:
\begin{equation}
\label{eq:sparse-attn}
\mathbf{y}_{k}(\mathbf{q}) 
= \sum_{i \in S_k} \alpha_i \mathbf{v}_i\,,~~~\text{where}~ \alpha_i=\frac{\exp(\mathbf{k}_i^\top\qb)}{\sum_{j\in\Scal_k}\exp(\mathbf{k}_j^\top\qb)}~~\text{for}~i\in\mathcal{S}.
% \quad \text{where} \quad
% \alpha_i = \frac{\exp( \mathbf{k}_i^\ts\mathbf{q})}
% {\sum_{j=1}^{n} \exp(\mathbf{k}_j^\top \mathbf{q})}.
\end{equation}
% \nicefrac{\exp(\mathbf{k}_i^\top \mathbf{q})}

\begin{figure*}[t]
    \centering
    \begin{minipage}{0.32\linewidth}
        \centering
        \includegraphics[width=\linewidth]{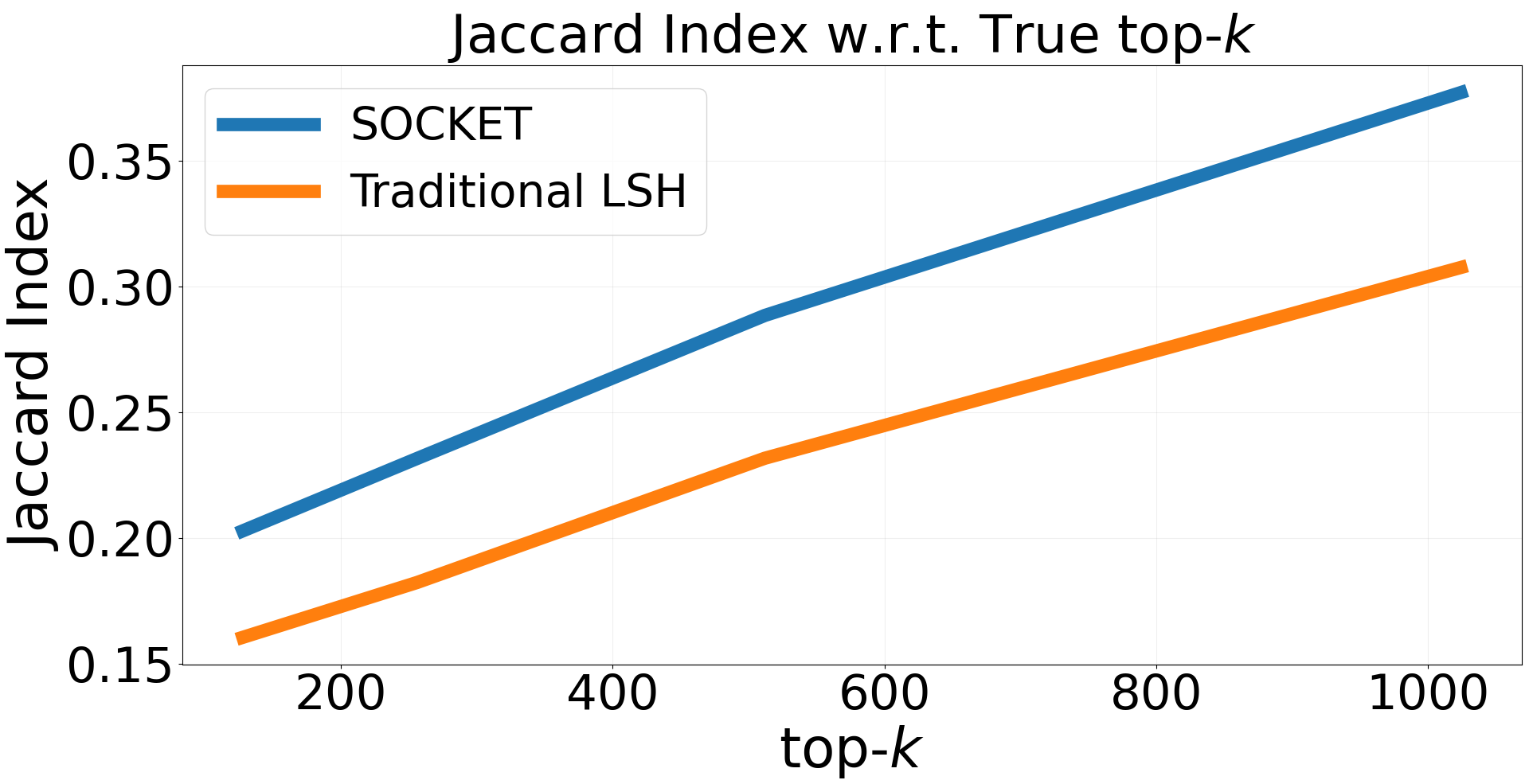}
        \caption*{Jaccard}
    \end{minipage}\hfill
    \begin{minipage}{0.32\linewidth}
        \centering
        \includegraphics[width=\linewidth]{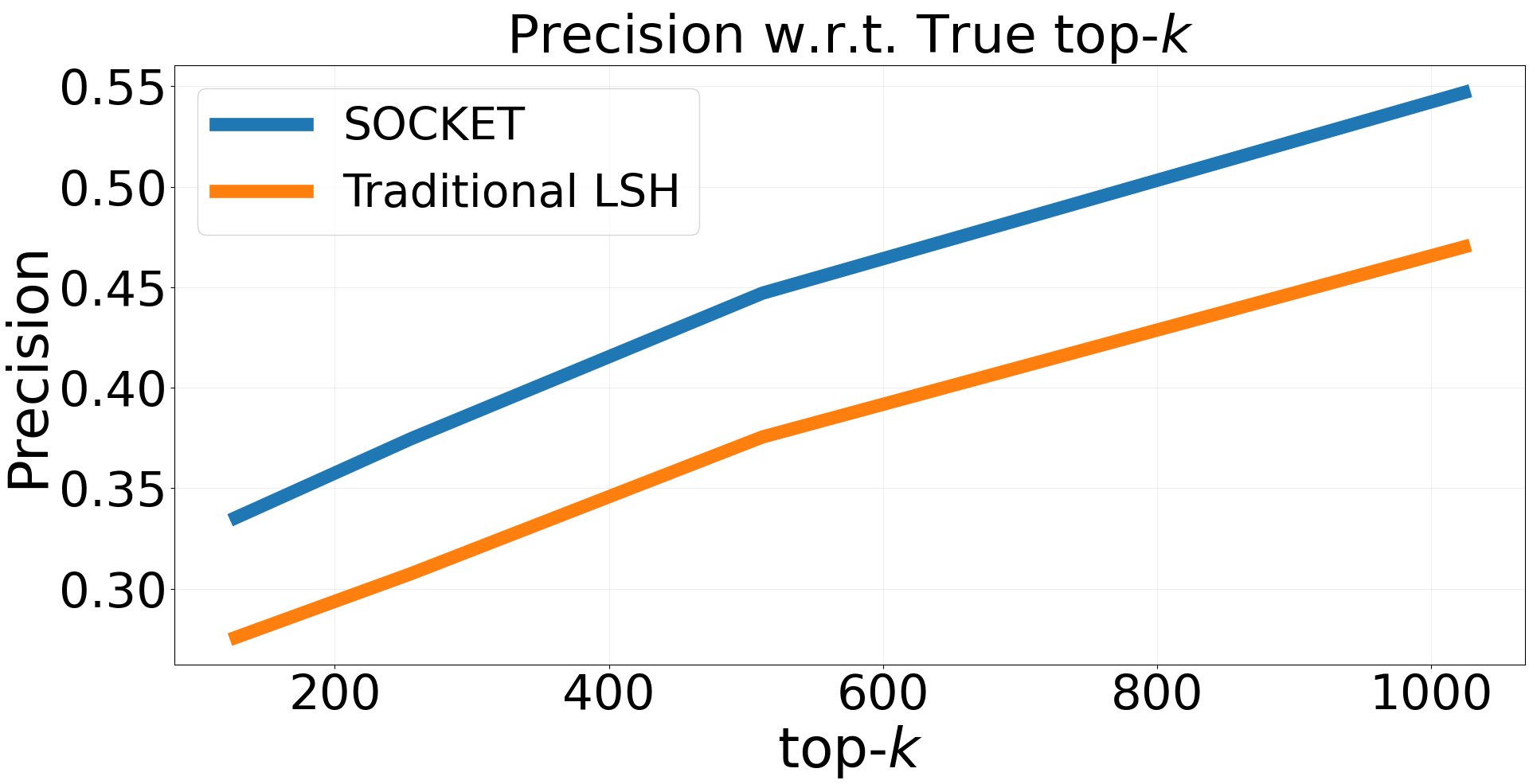}
        \caption*{Precision}
    \end{minipage}\hfill
    \begin{minipage}{0.32\linewidth}
        \centering
        \includegraphics[width=\linewidth]{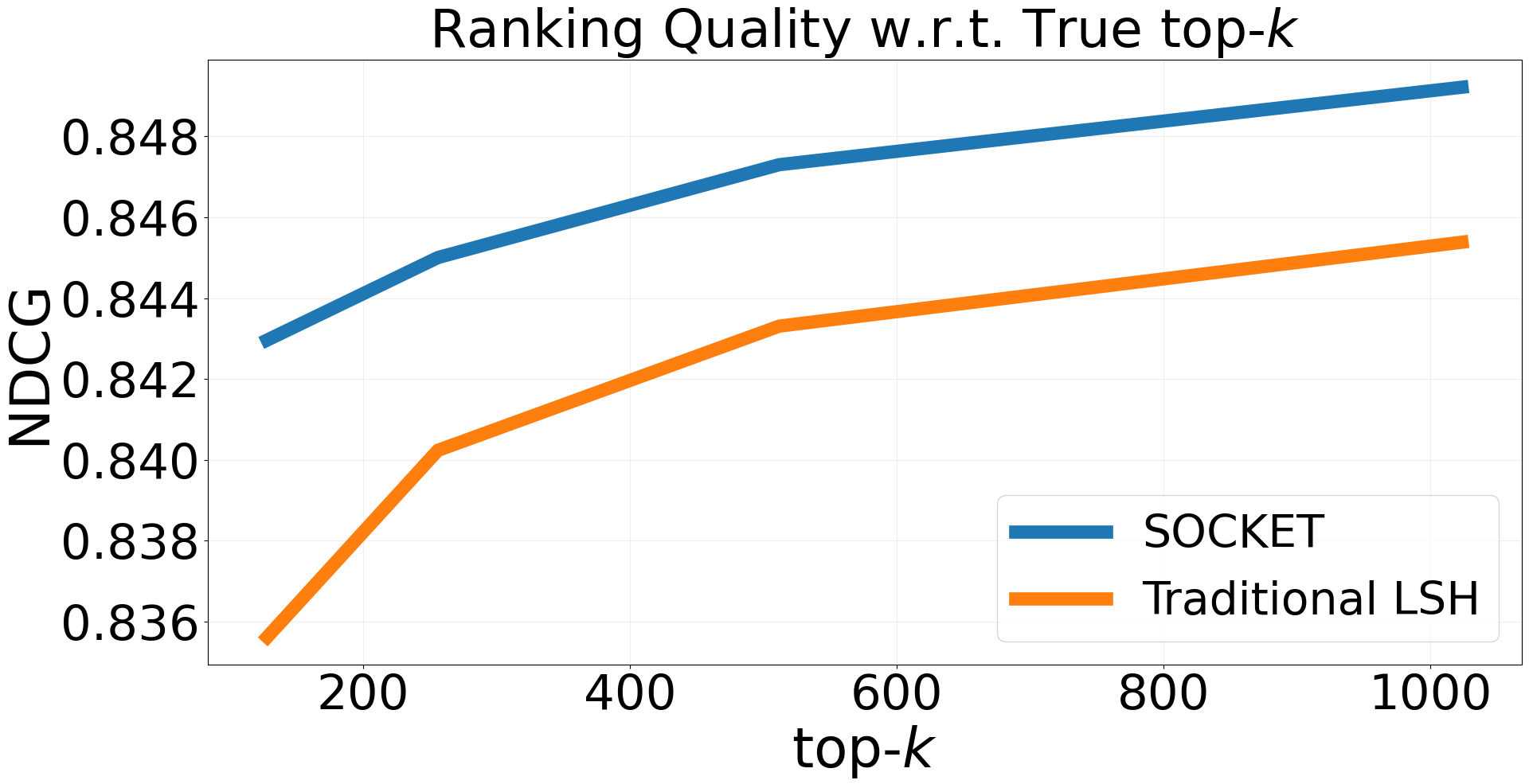}
        \caption*{NDCG}
    \end{minipage}

    \caption{Ranking quality comparison between SOCKET and traditional LSH under same memory budget (600 bits/token) as a function of top-$k$. Keys and queries are extracted from the final layer of Llama-3.1-8B on the Qasper dataset, and the ground-truth relevance is defined by the dot-product similarity between each key–query pair. Precision and Jaccard measure overlap with the ground-truth top-$k$ set, while NDCG additionally evaluates agreement with the ground-truth ranking. These metrics are well defined in the Appendix~\ref{sec:metric_def}.}
    \label{fig:ranking_metrics}
\end{figure*}

\section{Introducing SOCKET}

In dense attention, a query vector $\qb$ interacts with all $N$ keys to produce
attention weights, which are then applied to the corresponding values to generate
the output embedding. Prior work has shown that, for token generation, a small
subset of keys often dominates the attention mass~\citep{ge2024model, zhang2023h2o}.
We leverage this observation by using soft LSH as a sparsification mechanism,
treating it as a scoring function to identify the most relevant keys.
\\
\noindent
In traditional LSH-based retrieval, each key $\kb_j$ is projected into $L$
independent hash tables. For each table $\ell$, the key is assigned to a discrete
bucket {$b_j^{(\ell)}$}. At inference time, traditional LSH hashes the query $\qb$
in the same manner and scores a key by counting the number of hash tables in which
it collides with the query. We denote this hard collision score by
$s_{\mathrm{hard}}(\kb_j,\qb)$. In contrast, SOCKET replaces the hard query assignment
with a soft assignment: while keys are still assigned to a single bucket per hash
table, the query distributes probability mass across all buckets. We denote the
resulting SOCKET soft collision score by $s_{\mathrm{soft}}(\kb_j,\qb)$. 
These two scores are defined as
\begin{equation}
\label{eq:hard-soft-scores}
s_{\mathrm{hard}}(\kb_j,\qb)
=
\sum_{\ell=1}^{L}
\mathbb{I}\!\left[b_j^{(\ell)}=b_{\qb}^{(\ell)}\right],
\qquad
s_{\mathrm{soft}}(\kb_j,\qb)
=
\sum_{\ell=1}^{L}
p_\tau^{(\ell)}\!\left(b_j^{(\ell)}\mid \qb\right).
\end{equation}
SOCKET assigns higher scores to keys whose hash buckets receive greater probability mass under the query’s soft hash. We use these scores to select the top-$k$ keys, on which attention is subsequently computed. A clean schematic that distinguishes traditional LSH and soft LSH as a stable ranker can be seen in fig.~\ref{fig:soft-vs-hard}. 

\subsection{The Final Algorithm}

\begin{figure}[t]
\centering

\begin{minipage}[t]{0.49\linewidth}
\vspace{0pt}
\footnotesize
\begin{algorithm}[H]
  \caption{ PrecomputeKeyHashes (prefill)}
  \label{alg:race-precompute}
  \begin{algorithmic}[1]

    \State \textbf{Input:} $\{\kb_j, \vb_j\}_{j=1}^N$, \#planes $P$, \#tables $L$
    \State $R \gets 2^P$
    \For{$\ell=1$ to $L$}
      \State Sample $\Wb^{(\ell)} \sim \mathcal{N}(0,1)$
      \For{$j=1$ to $N$}
        \State $h^{(\ell)}(\kb_j)\gets \mathrm{sign}(\Wb^{(\ell)}\kb_j)$
        % \State $b_j^{(\ell)} \gets \mathrm{assign}(h^{(\ell)}(\kb_j))$
        \State Encode $h^{(\ell)}(\kb_j)$ as BucketId $b_j^{(\ell)}$
      \EndFor
    \EndFor
    \State \textbf{return} $\{||\vb_j||_2\}, \{\Wb^{(\ell)}\}, \{b_j^{(\ell)}\}$
    % \State \textbf{return} $\|\vb_j\|_2, \Wb^{(\ell)}, b_i^{(\ell)}$ for $i\in [R], j\in[N],\ell\in[L]$
  \end{algorithmic}
\end{algorithm}
\end{minipage}
\hfill
\begin{minipage}[t]{0.49\linewidth}
\vspace{0pt}
\centering
\scriptsize
\setlength{\tabcolsep}{4pt}

\captionof{table}{Correlation and variance comparison on \textsc{Samsum} and \textsc{Qasper}.}
\vspace{2pt}

\begin{tabular}{cc cc cc cc}
\toprule
& & \multicolumn{2}{c}{\textsc{Samsum}} & \multicolumn{2}{c}{\textsc{Qasper}} \\
\cmidrule(r){3-4}\cmidrule(l){5-6}
P & L & Corr & Var & Corr & Var \\
\midrule
\multicolumn{6}{c}{\textbf{SOCKET} ($\tau=0.5$)} \\
\midrule
10 & 20 & 0.504 & $8.9{\times}10^{-9}$ & 0.483 & $7.7{\times}10^{-9}$ \\
10 & 40 & 0.625 & $4.4{\times}10^{-9}$ & 0.606 & $4.6{\times}10^{-9}$ \\
10 & 60 & 0.690 & $3.03{\times}10^{-9}$ & 0.673 & $2.7{\times}10^{-9}$ \\
\midrule
\multicolumn{6}{c}{\textbf{Hard LSH}} \\
\midrule
2 & 250 & 0.573 & $5{\times}10^{-4}$ & 0.550 & $5{\times}10^{-4}$ \\
2 & 300 & 0.605 & $4{\times}10^{-4}$ & 0.583 & $4{\times}10^{-4}$ \\
2 & 350 & 0.633 & $3{\times}10^{-4}$ & 0.610 & $3{\times}10^{-4}$ \\
\bottomrule
\end{tabular}

\label{tab:corr_variance_combined}
\end{minipage}

\end{figure}

\begin{figure}[t]
\centering

% -------- LEFT: TABLE --------
\begin{minipage}[t]{0.49\linewidth}
\vspace{0pt}
\footnotesize
\begin{algorithm}[H]
\caption{SoftBucketProbs (decoding)}
\label{alg:soft-bucket-probs}
\begin{algorithmic}[1]
\State \textbf{Input:} $\qb$, $\{\Wb^{(\ell)}\}$, 
$\{\mathbf{c}_r\}_{r=1}^{R}\subset\{\pm1\}^P$,
% $\{\mathbf{c}_r\}$, 
$\tau$
\For{$\ell=1$ to $L$}
  \State $\mathbf{u}^{(\ell)}(\qb) \gets \tfrac{1}{\sqrt{d}}\tanh(\Wb^{(\ell)}\qb)\in\mathbb{R}^P$
  \For{$r=1$ to $R$}
    \State $\text{logit}^{(\ell)}(r) \gets \frac{\mathbf{u}^{(\ell)}(\qb)^\top \mathbf{c}_r}{\tau}$
  \EndFor
  \State $p_{\tau}^{(\ell)}(\cdot\mid \qb)\gets \mathrm{softmax}(\text{logits})$
\EndFor
\State \textbf{return} $p_{\tau}^{(\ell)}(r\mid \qb)$ for all $r\in[R]$ and $\ell\in [L]$.
\end{algorithmic}
\end{algorithm}
\end{minipage}
\hfill
% -------- RIGHT: ALGORITHM --------
\begin{minipage}[t]{0.45\linewidth}
\vspace{0pt}
\footnotesize

\begin{algorithm}[H]
\caption{ValueAwareTop-$k$ Attention}
\label{alg:value-aware-topk}
\begin{algorithmic}[1]
\State \textbf{Input:} $\qb$, $p_{\tau}^{(\ell)}$, $b_j^{(\ell)}$, values $\vb_j$, $L$, $k$

\For {$j=1$ to $N$}
  \State $\widehat w_j \gets \sum_{\ell=1}^L
  p_{\tau}^{(\ell)}(b_j^{(\ell)}\mid \qb)$
\EndFor

\vspace{1mm}
\State $\mathcal{S}_k \gets \operatorname{TopK}(\widehat w_1\|\vb_1\|_2,\dots,\widehat{w}_N\|\vb_N\|_2)$

% \For {$j \in \mathcal{S}_k$}
%   \State $\alpha_j \gets
%   \nicefrac{\exp(\kb_j^\top \qb)}
%   {\sum_{i\in\mathcal{S}_k}\exp(\kb_i^\top \qb)}$
% \EndFor
\State $\alpha_j \gets
\frac{\exp(\kb_j^\top \qb)}{\sum_{i\in\mathcal{S}_k}\exp(\kb_i^\top \qb)}$ \quad
for all $j\in\mathcal{S}_k$
\vspace{2mm}

% \vspace{1mm}
\State \textbf{return} $\mathbf{y}_{k}(\mathbf{q}) 
= \sum_{j\in\mathcal{S}_k} \alpha_j \vb_j$
\end{algorithmic}
\end{algorithm}

\end{minipage}

\end{figure}

Our method consists of three stages: 
\\
\noindent
\textbf{Algorithm~\ref{alg:race-precompute} (Keys hashing):} During the prefill phase, each key vector is hashed into $L$ independent hash tables. Specifically, we apply $L$ random projections and assign each key to a single bucket per hash table. These bucket assignments are computed once and cached in GPU memory for reuse during decoding, together with the corresponding value vector norms. \\
\textbf{Algorithm~\ref{alg:soft-bucket-probs} (Query soft hashing):} At decoding time, a new query vector is softly hashed into the same $L$ hash tables. Instead of a single bucket assignment, the query induces a probability distribution over buckets in each table. Using these distributions, we compute a score for every key by aggregating the probability mass assigned to the key’s corresponding buckets, as defined in \eqref{eq:hard-soft-scores}. \\
\textbf{Algorithm~\ref{alg:value-aware-topk} (Top-$k$ selection):} Finally, we select the top-$k$ keys according to their soft LSH scores and compute attention using only this subset. This yields a sparse yet high-quality approximation of dense attention. We discuss approximation quality in the next section.

\section{Theoretical Insights}
In this section, we present the theoretical analysis of our method.
We characterize the quality of the similarity scores produced by soft LSH and study a
sampling-based estimator to establish concentration and approximation guarantees with
respect to cosine similarity–based (angular) attention.
We adopt angular attention as a modeling choice, as it provides an interpretable and analytically
tractable surrogate for the attention scoring function. Concretely, for a fixed query $\qb$ and keys $\{\kb_j\}_{j=1}^N$, we define the
angular kernel weights
\begin{equation}
w_j
\;:=\;
\left(
1 - 
\frac{1}{\pi}\cos^{-1}\!\left(
\frac{\qb^\top \kb_j}{\|\qb\|\,\|\kb_j\|}
\right)
\right)^P
\in [0,1],
\label{eq:angular-kernel}
\end{equation}
where $P$ controls the sharpness of the kernel.
Normalizing these weights yields the angular attention distribution 
and the corresponding angular attention output is given by
\[
a_j:=\frac{w_j}{Z},
\qquad
Z := \sum_{i=1}^N w_i,
\qquad
\yb^*
:=
\sum_{j=1}^N a_j \vb_j.
\]
% and the corresponding angular attention output
% \begin{equation}
% \yb^*
% \;:=\;
% \sum_{j=1}^N a_j \vb_j.
% \label{eq:angular-attention-output}
% \end{equation}
%%%%%%%%%%%%%%%%%%%%%%%%%%%%%%%%%%%%%%%%%%%%%%%%%%%
Our analysis targets an approximation of $\yb^*$. Prior work shows that cosine-similarity kernels amenable to LSH closely approximate softmax attention in practice~\cite{joshi2026race}, so we analyze this surrogate. Sampling is used only for analysis, enabling unbiased estimators and high-probability error bounds; the inference procedures in Algorithms~\ref{alg:race-precompute} and~\ref{alg:soft-bucket-probs} remain unchanged, with only the final aggregation replaced by a sampling-based estimator.  In practice (Section~\ref{sec:exp}), Algorithm~\ref{alg:value-aware-topk} performs deterministic top-$k$ selection using the same soft LSH scores, followed by standard softmax over the selected subset. Thus, the implementation preserves the pre-trained attention mechanism after retrieval, while the theory analyzes normalized soft LSH scores as a proxy attention distribution to study the SOCKET scoring kernel. We leave empirical evaluation of the sampling approach for future work.
%
% In this context, we also emphasize that sampling is introduced only to meet the underlying sparsity budget rather than to enable the approximation itself. Even without sampling, the finite-table soft collision estimator $\yb_{\tau,L}(\qb)$ (defined below) provides a principled approximation to angular attention, with accuracy controlled by $L$ and the soft-bucketization bias.

\subsection{Sampling-based Estimator}
Our sampling-based estimator is constructed from the same scores produced in
Algorithm~\ref{alg:value-aware-topk}, namely the soft-LSH weights
$\widehat w_1,\dots,\widehat w_N$.
For the purpose of theoretical analysis, we rescale these scores
\emph{by the number of hash tables} $L$ to obtain {$\widetilde w_j
\;:=\;
\frac{1}{L}\sum_{\ell=1}^L
p_{\tau}^{(\ell)}\!\big(b_j^{(\ell)} \mid \qb\big)
=\
\frac{1}{L}\,\widehat w_j$ for $j=1,\dots,N$}. To interpret the rescaled scores as an attention distribution, we further normalize $\widetilde a_j
:=
\nicefrac{\widetilde w_j}{\widetilde Z}$ with $\widetilde Z
:=
\sum_{j=1}^N \widetilde w_j$ for $j=1,\dots,N$. The normalized weights \(\widetilde a_j\) form a probability distribution over keys and serve
as our proxy for angular attention coefficients in the theoretical analysis and the resulting attention output
\begin{equation}
\yb_{\tau,L}(\qb)
\;=\;
\sum_{j=1}^N \widetilde{a}_j \vb_j.
\end{equation}
Note that the randomness in $\yb_{\tau,L}(\qb)$ is due to the random hyperplanes for each hash table that we introduced. Consequently, we define the population (single-table) soft-count weight $w_{\tau, j}:=\mathbb{E}\big(s_j^{(1)}(\qb)\big)$, the denominator $Z_\tau:=\sum_{j=1}^N w_{\tau, j}$, and the output $\yb_\tau(\qb):=\sum_{j=1}^N a_{\tau, j} \vb_j$, where $a_{\tau, j}=w_{\tau, j}/Z_\tau$.
Now, we define a sampling distribution $p_j \;\propto\; \widetilde a_j\|\vb_j\|_2$,
and consider an estimator formed by drawing $M$ independent samples
$J_1,\dots,J_M \sim p:=(p_1,\dots,p_N)$ and aggregating
\begin{flalign}
\mathbf{T}(\qb)
:=
\frac{1}{M}\sum_{m=1}^M
\frac{\widetilde a_{J_m}}{p_{J_m}}\,\vb_{J_m}.\label{eq:est}
\end{flalign}
In the following section, we analyze $\mathbf{T}(\mathbf{q})$ using concentration arguments from Randomized Numerical Linear Algebra (RandNLA) literature\,\cite{intro2015tropp,sketching2014woodruff} 
% Specifically, 
and bound its deviation from angular attention.

\subsection{Theoretical Analysis}\label{sec:assumptions}

\begin{assumption}\label{as1}
 $Z\ge Z_{\text{min}}>0$ and $Z_\tau\ge Z_{\tau,\text{min}}>0$.   
\end{assumption}

\vspace{-5mm}
\begin{assumption}\label{as2}
For each table $\ell\in[L]$ and bucket $r\in[R]$, the number of elements are bounded \emph{i.e.,} $$\#\left\{j\in[N]: b_j^{(\ell)}=r\right\}\le B.$$    
\end{assumption}

\vspace{-5mm}
\begin{theorem}
\label{thm:main}
Let $\qb\in\mathbb{R}^d$ be a fixed query, with keys
$\kb_1,\dots ,\kb_N\in\mathbb{R}^d$ and values
$\vb_1,\dots,\vb_N\in\mathbb{R}^d$. 
Under Assumptions~\eqref{as1} and~\eqref{as2}, and for parameters $L$, $M$, and $\tau$
with
$L\ge\frac{2B^2\log(8/\delta)}{Z_{\tau,\min}^2}$,
the estimator $\mathbf{T}(\qb)$ in eq.~\eqref{eq:est}
satisfies
{
\begin{flalign}
\|\mathbf{T}(\qb)-\yb^*(\qb)\|_2
% =\mathcal{O}\!\left(
% \sqrt{\frac{\log(1/\delta)}{M}}
% +\sqrt{\frac{\log(1/\delta)}{L}}
% +\varepsilon_\tau(\qb)
% \right)\|\mathbf{V}\|_2,
=\widetilde{\mathcal{O}}\!\left(
\frac{1}{\sqrt{L}}
+\frac{1}{\sqrt{M}}
+\varepsilon_\tau(\qb)
\right)\|\mathbf{V}\|_2,
\nonumber
\end{flalign}
}
\noindent
with probability at least $1-\delta$, where $\yb^*(\qb)$ denotes the target (angular) attention output.
Here $\widetilde{\mathcal{O}}(\cdot)$ hides a $\sqrt{\log(1/\delta)}$ factor and absolute constants depending only on
$(B, Z_{\min}, Z_{\tau,\min})$, and is independent of $N,d,L,$ and $M$.
Moreover, $\varepsilon_\tau(\qb)$ quantifies the bias introduced by soft bucketization; it depends on
$\tau$ and $P$ (equivalently $R=2^P$) and satisfies $\varepsilon_\tau(\qb)\to 0$ as $\tau\to 0$ for fixed $P$.
\end{theorem}
\noindent
Theorem~\ref{thm:main} provides an end-to-end error decomposition for our proposed
soft-count attention estimator $\Tb(\qb)$.
The bound separates the total error into three terms corresponding to
sampling variance, finite-table approximation error, and a bias induced by
soft bucketization given by $\varepsilon_\tau=\EE[1-p_\tau^{(\ell)}(b_q\mid \qb)]$.
This decomposition makes explicit how algorithmic parameters control different
sources of error. 
% See Section~\ref{thm:proof} for additional discussion.
Due to space constraints, we defer the proof of Theorem~\ref{thm:main} and further discussion of the assumptions, parameter roles ($L$, $M$ and $\tau$), and the analysis of $\yb_{\tau,L}(\qb)$ to Appendix~\ref{thm:proof}.

\begin{figure*}[htbp]
    \centering
    \begin{subfigure}{0.32\linewidth}
        \centering
        \includegraphics[width=\linewidth]{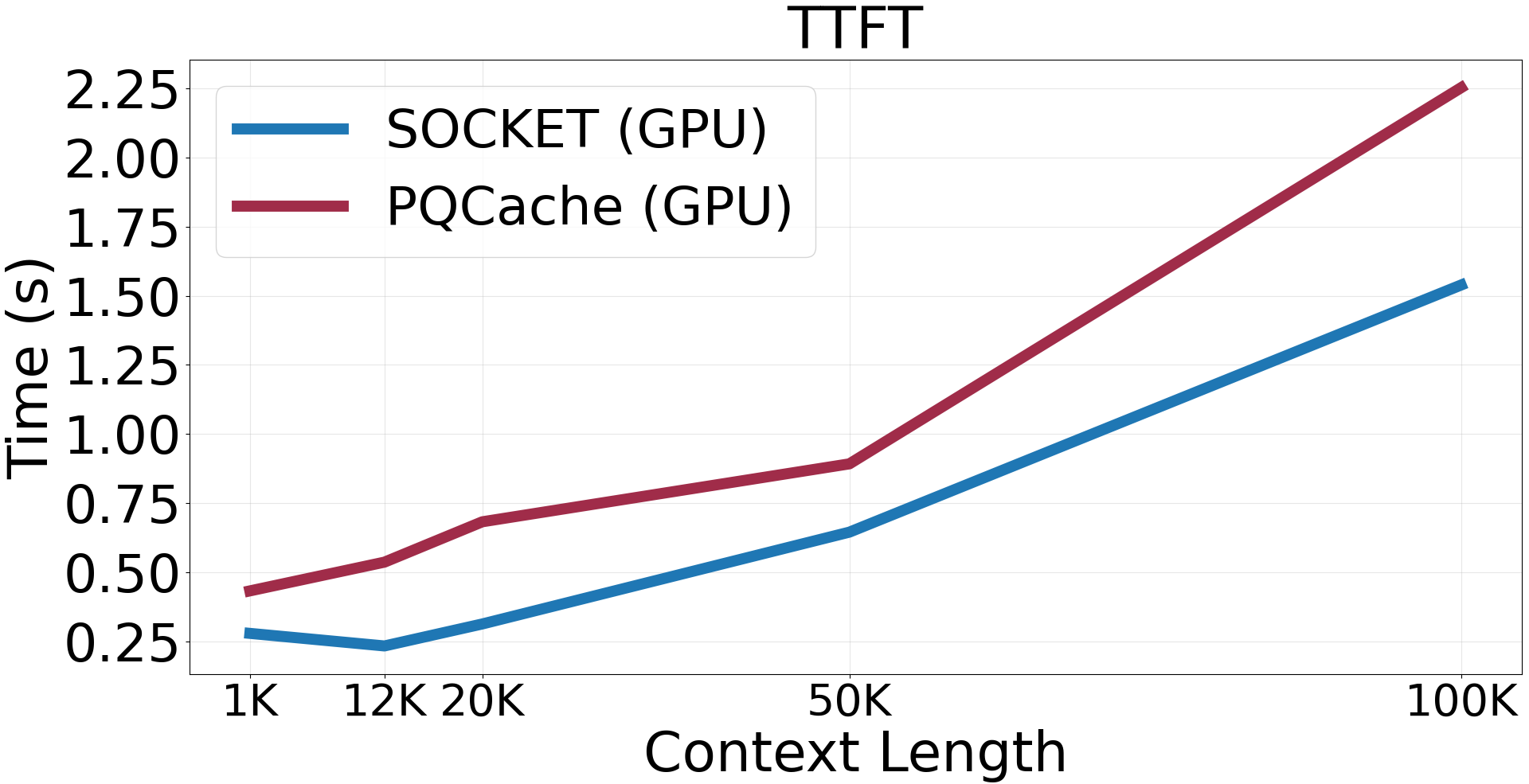}
        \caption{TTFT comparison}
        \label{fig:pq_gpu}
    \end{subfigure}
    \hfill
        \begin{subfigure}{0.32\linewidth}
        \centering
        \includegraphics[width=\linewidth]{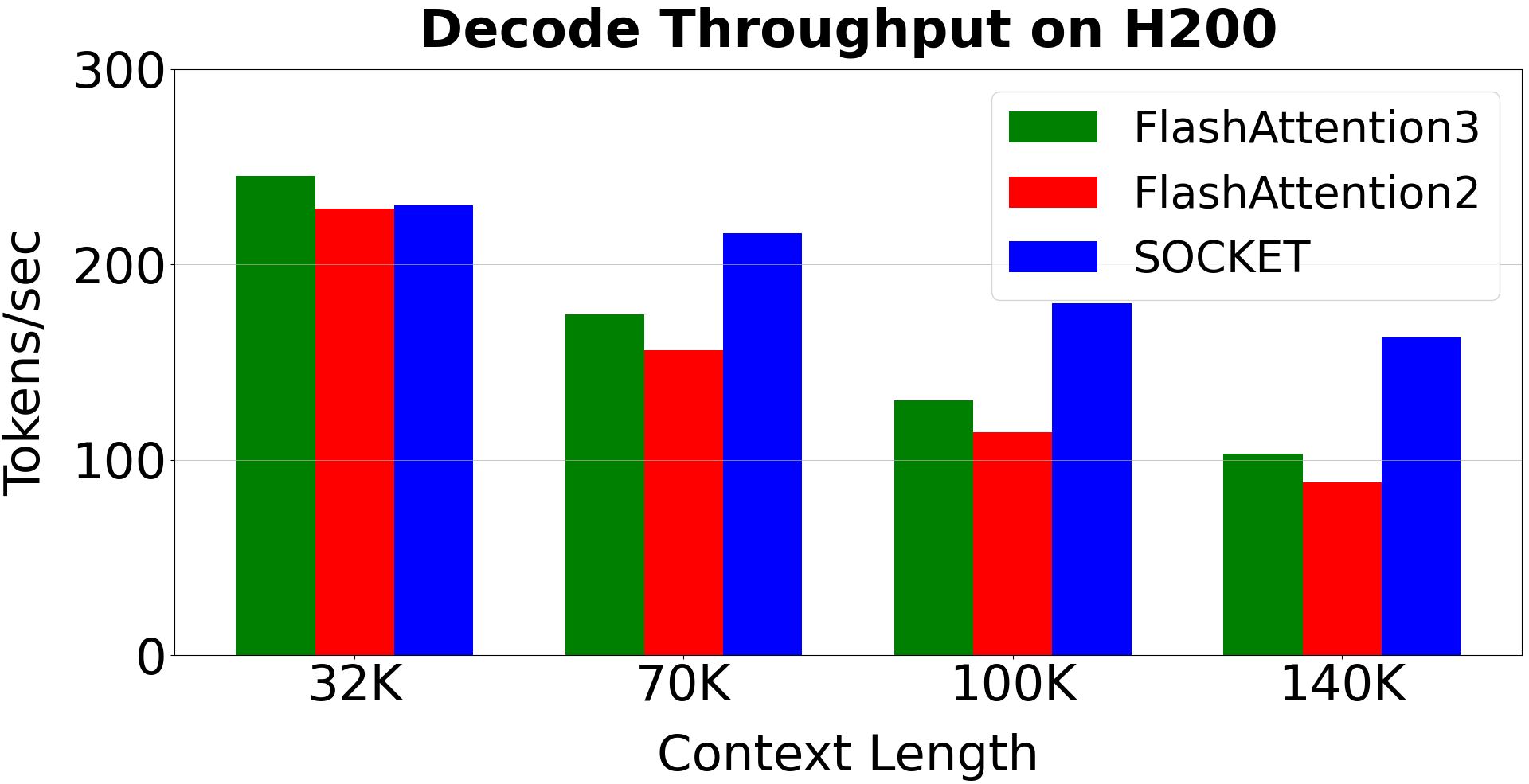}
        \caption{Throughput increase on H200}
        \label{fig:decode_h200}
    \end{subfigure}
    \hfill
    \begin{subfigure}{0.32\linewidth}
        \centering
        \includegraphics[width=\linewidth]{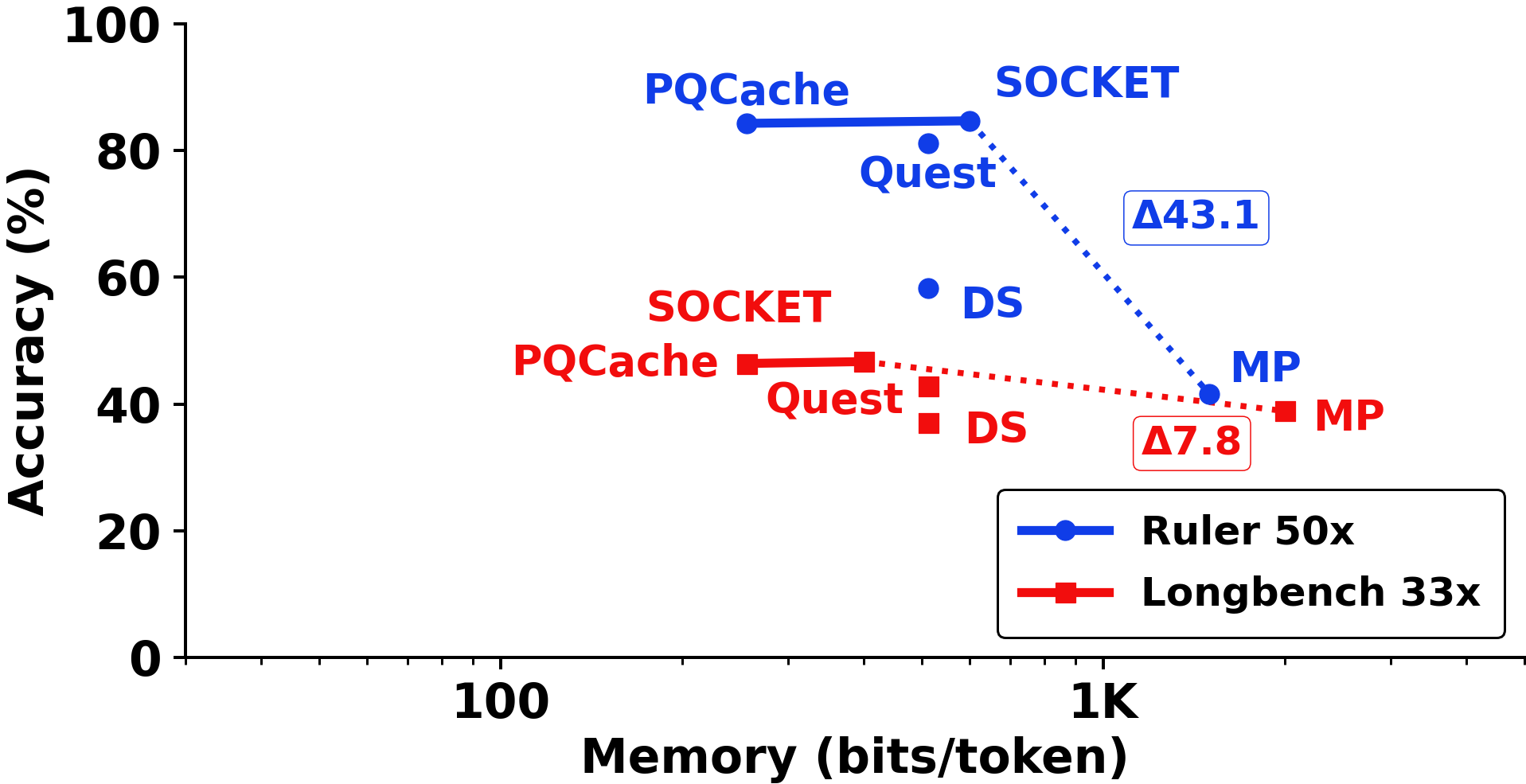}
        \caption{Pareto Curve}
        \label{fig:pareto}
    \end{subfigure}
    \caption{(a) Time-to-first-token (TTFT) comparison on H100 for SOCKET and PQCache indexers. (b) Decode-only throughput as a function of context length for SOCKET (33$\times$ sparsity) and FlashAttention using GPT-FAST (Llama-3.1-8B). (c) Memory-Accuracy pareto curves with Ministral-3-8B.}
    \label{fig:overall}
\end{figure*}

\begin{table*}[t]
\centering
\caption{Comparison of dense and sparse attention methods on LongBench
(Llama-3.1-8B).}
\label{tab:llama-longbench}
\resizebox{\textwidth}{!}{
\begin{tabular}{lcccccccccccccccc}
\toprule
Method & Sparsity & NQA & QAS & MFQA & HPQA & WIKI & MUS & GOV & QMSUM & MNews & LCC & Trivia & SamSUM & Retrieval & Repo & Avg \\
\midrule

Baseline
& Dense
& 30.0 & 46.9 & 55.9 & 59.0 & 52.0 & 33.2 & 35.2 & 25.2 & 26.8 & 53.5 & 85.2 & 38.2 & 100 & 47.4 & \textbf{49.1} \\

\midrule

Quest (ps=32)
& 10$\times$
& 30.0 & 42.1 & 53.9 & 51.2 & 43.1 & 27.1 & 33.9 & 24.2 & 25.2 & 53.1 & 74.6 & 40.1 & 99.5 & 49.3 & 46.2 \\

DS ($l$=8, ch=32)
& 10$\times$
& 28.9 & 47.1 & 53.9 & 53.6 & 46.6 & 25.4 & 35.1 & 24.6 & 26.4 & 53.6 & 79.9 & 21.9 & 99.5 & 43.8 & 45.7 \\

HashAttn ($b$=128)
& 10$\times$
& 28.8 & 45.7 & 55.6 & 53.0 & 45.9 & 25.4 & 35.1 & 24.7 & 26.7 & 51.6 & 76.1 & 19.9 & 100 & 44.0 & 45.2 \\

PQCache (pq=8, gr=4)
& 10$\times$
& 28.8 & 45.9 & 54.0 & 58.7 & 52.0 & 27.2 & 35.0 & 25.2 & 26.8 & 55.5 & 76.9 & 26.4 & 99.5 & 44.3 & 46.8 \\

MagicPig (P=8, L=120)
& 10$\times$
& 26.8 & 38.2 & 45.4 & 49.1 & 35.1 & 23.3 & 30.7 & 24.3 & 25.1 & 48.6 & 73.1 & 32.8 & 85.6 & 43.0 & 41.5 \\

SOCKET (P=8, L=50)
& 10$\times$
& \textbf{30.5} & \textbf{48.5} & \textbf{56.2} & 56.9 & 49.3 & 30.0
& \textbf{35.6} & 25.0 & \textbf{27.0} & 54.5 & \textbf{82.6} & \textbf{42.6}
& \textbf{100.0} & \textbf{47.4} & \textbf{49.0} \\

\midrule

Quest (ps=16)
& 33$\times$
& 27.2 & 36.0 & 43.2 & 47.7 & 35.9 & 24.2 & 31.9 & 23.4 & 24.9 & 44.2 & 76.6 & 32.6 & 92.5 & 43.2 & 41.7 \\

Quest (ps=8)
& 33$\times$
& 30.1 & 42.4 & 54.9 & 53.6 & 42.5 & 24.5 & 34.0 & 23.7 & 26.2 & 46.4 & 72.1 & 38.5 & 98.0 & 45.3 & 45.1 \\

DS ($l$=16, ch=32)
& 33$\times$
& 26.8 & 38.5 & 47.5 & 48.7 & 43.2 & 19.2 & 32.6 & 23.4 & 25.3 & 42.2 & 68.9 & 30.7 & 97.0 & 39.1 & 41.6 \\

HashAttn ($b$=128)
& 33$\times$
& 27.0 & 42.2 & 52.4 & 48.9 & 38.5 & 21.6 & 34.7 & 23.6 & 26.9 & 33.9 & 73.7 & 15.2 & 94.0 & 35.8 & 40.6 \\

PQCache (pq=8, gr=4)
& 33$\times$
& 28.2 & 43.7 & 52.7 & 54.5 & 45.5 & 27.4 & 35.1 & 24.0 & 26.7 & 51.9 & 70.4 & 27.7 & 99.8 & 42.1 & 45.0 \\

MagicPig (P=10, L=200)
& 33$\times$
& 21.8 & 26.1 & 31.1 & 41.2 & 23.6 & 15.5 & 28.3 & 23.1 & 23.3 & 41.9 & 62.4 & 28.3 & 48.9 & 37.7 & 32.4 \\

SOCKET (P=8, L=50)
& 33$\times$
& \textbf{29.3} & 43.7 & 51.0 & 53.1$^{\dagger}$ & 45.2 & 26.0 & 35.0 & 24.0
& \textbf{27.0} & 47.7$^{\dagger}$ & \textbf{79.1} & \textbf{34.5}
& 99.3 & 41.4$^{\dagger}$ & \textbf{45.5} \\

\bottomrule
\end{tabular}
}
\end{table*}

\begingroup
\renewcommand{\thefootnote}{}
\footnotetext{%
\makebox[\textwidth][l]{%
    \parbox[t]{0.45\textwidth}{%
        $^{\dagger}$Configuration used - $(P=8, L=55, \tau=0.3).$%
    }%
    \hspace{0.25cm}
    \parbox[t]{0.45\textwidth}{%
        $^{\star}$ Configuration used - $(P=10, L=10, \tau=1e^{-3})$.%
    }%
}%
}

% \footnote{$^{\dagger}$ We use (P=8, L=55, $\tau=0.3$) as the config.}

\begin{table*}[t]
\centering
\caption{Comparison of dense and sparse attention methods on LongBench
(Qwen3-8B).}
\label{tab:qwen-longbench}
\resizebox{\textwidth}{!}{
\begin{tabular}{lcccccccccccccccc}
\toprule
Method & Sparsity & NQA & QAS & MFQA & HPQA & WIKI & MUS & GOV & QMSUM & MNews & LCC & Trivia & SamSUM & Retrieval & Repo & Avg \\
\midrule

Baseline
& Dense
& 26.5 & 43.8 & 55.7 & 62.3 & 48.0 & 35.3 & 33.5 & 24.5 & 24.6 & 65.8 & 90.3 & 40.0 & 91.6 & 62.2 & \textbf{50.3} \\

\midrule

Quest (ps=32)
& 10$\times$
& 24.4 & 37.3 & 51.3 & 57.6 & 44.7 & 33.1 & 32.7 & 23.6 & 23.2 & 65.8 & 85.8 & 38.8 & 95.0 & 62.7 & 48.3 \\

DS ($l$=8, ch=32)
& 10$\times$
& 22.1 & 35.0 & 46.6 & 53.3 & 43.5 & 25.9 & 31.2 & 22.2 & 23.2 & 67.7 & 85.7 & 38.3 & 95.6 & 60.0 & 46.5 \\

PQCache (pq=8, gr=4)
& 10$\times$
& 26.0 & 43.0 & 53.8 & 63.3 & 48.7 & 33.5 & 33.8 & 24.2 & 24.6 & 68.1 & 89.6 & 42.3 & 94.5 & 62.3 & 50.6 \\

MagicPig (P=8, L=120)
& 10$\times$
& 25.5 & 41.4 & 52.1 & 58.4 & 43.9 & 31.1 & 32.7 & 23.9 & 24.2 & 65.5 & 84.4 & 39.6 & 94.4 & 58.8 & 48.3 \\

SOCKET (P=8, L=50)
& 10$\times$
& 24.0 & 42.2 & \textbf{55.7} & \textbf{63.5} & 48.5
& \textbf{35.0} & \textbf{34.5} & 24.0 & \textbf{24.7}
& 67.3 & 88.8 & 41.8 & \textbf{98.7} & 61.9 & \textbf{50.8} \\

\midrule

Quest (ps=16)
& 33$\times$
& 22.1 & 26.3 & 43.5 & 51.1 & 35.2 & 27.4 & 30.4 & 21.4 & 22.3 & 50.3 & 77.8 & 35.8 & 85.2 & 50.1 & 41.4 \\

Quest (ps=8)
& 33$\times$
& 26.6 & 37.9 & 51.8 & 58.0 & 45.9 & 28.4 & 32.0 & 23.1 & 22.9 & 55.2 & 80.5 & 37.4 & 93.3 & 57.0 & 46.4 \\

DS ($l$=16, ch=32)
& 33$\times$
& 11.3 & 16.9 & 31.2 & 30.1 & 25.7 & 8.3 & 24.2 & 18.5 & 21.0 & 44.6 & 57.5 & 26.3 & 27.1 & 40.8 & 27.4 \\

PQCache (pq=8, gr=4)
& 33$\times$
& 24.7 & 39.4 & 52.6 & 61.6 & 47.4 & 30.0 & 34.1 & 24.0 & 24.5 & 66.1 & 84.0 & 42.1 & 92.4 & 59.1 & 48.7 \\

MagicPig (P=10, L=200)
& 33$\times$
& 21.5 & 36.7 & 49.6 & 52.4 & 38.9 & 28.1 & 32.4 & 23.6 & 23.7 & 61.1 & 75.3 & 36.7 & 90.4 & 53.4 & 44.6 \\

SOCKET (P=8, L=50)
& 33$\times$
& 23.6 & 39.4 & 52.3 & 61.0
& 47.2$^\ddagger$
& \textbf{32.6} & \textbf{34.8} & 23.6 & \textbf{26.0}
& 61.9$^\ddagger$
& \textbf{86.3}
& 40.8$^\ddagger$
& \textbf{98.8}
& 54.8$^\ddagger$
& \textbf{49.2} \\

\bottomrule
\end{tabular}
}
\end{table*}

\begingroup
\renewcommand{\thefootnote}{}
\footnotetext{%
\makebox[\textwidth][l]{%
    \parbox[t]{0.45\textwidth}{%
        $^{\ddagger}$Configuration used - $(P=10, L=70, \tau=0.3).$%
    }%
    \hspace{0.25cm}
    \parbox[t]{0.45\textwidth}{%
        $^{\parallel}$ Configuration used - $(P=10, L=70, \tau=0.3)$.%
    }%
}%
}
\endgroup

\begin{table*}[t]
\centering
\caption{Comparison of dense and sparse attention methods on LongBench
(Ministral-3-8B)}
\label{tab:ministral-longbench}
\resizebox{\textwidth}{!}{
\begin{tabular}{lcccccccccccccccc}
\toprule
Method & Sparsity & NQA & QAS & MFQA & HPQA & WIKI & MUS & GOV & QMSUM & MNews & LCC & Trivia & SamSUM & Retrieval & Repo & Avg \\
\midrule

Baseline
& Dense
& 26.9 & 46.2 & 44.6 & 61.0 & 54.9 & 41.0 & 24.0 & 17.8 & 20.5 & 59.5 & 90.8 & 37.2 & 99.7 & 61.3 & 49.0 \\

\midrule

Quest (ps=32)
& 10$\times$
& 27.3 & 43.8 & 42.1 & 58.4 & 50.7 & 35.9 & 23.5 & 18.6 & 19.4 & 41.1 & 88.1 & 34.5 & 100.0 & 46.0 & 45.0 \\

DS ($l$=8, ch=32)
& 10$\times$
& 25.1 & 43.7 & 42.9 & 56.2 & 51.4 & 39.0 & 23.6 & 18.0 & 19.6 & 46.5 & 88.3 & 36.0 & 100.0 & 43.0 & 45.2 \\

PQCache (pq=8, gr=4)
& 10$\times$
& 27.4 & 44.2 & 44.1 & 60.1 & 55.8 & 37.5 & 24.4 & 18.4 & 20.1 & 60.0 & 89.8 & 37.8 & 99.9 & 61.4 & 48.6 \\

MagicPig (P=8, L=120)
& 10$\times$
& 25.9 & 41.6 & 40.8 & 54.2 & 47.4 & 36.0 & 23.8 & 17.2 & 19.8 & 52.5 & 88.3 & 36.4 & 99.4 & 54.9 & 45.6 \\

SOCKET (P=8, L=50)
& 10$\times$
& 26.6 & 46.2 & 45.0 & 59.5 & 56.2 & 37.1 & 24.5 & 18.6 & 20.5 & 58.6 & 89.6 & 37.9 & 99.9 & 61.7 & \textbf{48.7} \\

\midrule

Quest (ps=16)
& 33$\times$
& 23.1 & 38.6 & 33.2 & 38.4 & 36.0 & 28.2 & 22.6 & 17.9 & 20.1 & 37.8 & 85.0 & 31.1 & 97.9 & 28.0 & 38.4 \\

Quest (ps=8)
& 33$\times$
& 27.3 & 43.1 & 42.7 & 56.0 & 48.2 & 27.3 & 23.7 & 18.5 & 20.4 & 42.8 & 88.6 & 33.6 & 99.7 & 28.6 & 42.9 \\

DS ($l$=16, ch=32)
& 33$\times$
& 23.2 & 37.0 & 36.9 & 43.0 & 30.2 & 26.5 & 24.2 & 16.6 & 19.0 & 26.3 & 83.1 & 34.0 & 90.5 & 28.8 & 37.1 \\

PQCache (pq=8, gr=4)
& 33$\times$
& 24.8 & 45.1 & 39.2 & 54.3 & 50.5 & 34.0 & 24.7 & 18.7 & 20.1 & 56.6 & 88.0 & 37.2 & 100.0 & 56.2 & 46.4 \\

MagicPig (P=10, L=200)
& 33$\times$
& 22.5 & 33.8 & 32.7 & 45.7 & 27.8 & 24.7 & 23.0 & 17.3 & 19.0 & 46.7 & 82.0 & 34.6 & 88.1 & 46.8 & 38.9 \\

SOCKET (P=8, L=50)
& 33$\times$
& 24.3 & 46.2 & 45.5 & 56.9 & 51.3 & 36.0 & 24.4 & 18.8 & 21.0
& 49.3$^\parallel$ & 89.0 & 37.8 & 99.9 & 53.2$^\parallel$ & \textbf{46.7} \\

\bottomrule
\end{tabular}
}
\end{table*}

\vspace{-5pt}
\subsection{Why Soft Collisions Yield More Stable Rankings than Hard LSH?}
\vspace{-1.2mm}

Our inference procedure relies on top-$k$ selection from aggregated collision scores.
Thus, the relevant criterion is ranking stability at finite $L$. When two keys have
comparable scores, even small per-table fluctuations can change their ordering and
lead to incorrect retrieval.
Theorem~\ref{thm:main} shows that soft bucketization introduces a controllable bias
$\varepsilon_\tau(q)$, while the finite-table and sampling errors decay as
$L^{-1/2}$ and $M^{-1/2}$. In the limit $\tau\to 0$, this bias vanishes and SOCKET
recovers hard LSH. However, this limit also replaces smooth scores by discrete
collision indicators, which are less informative for ranking at finite $L$.

% Hard collisions are maximally noisy;
\noindent
\textbf{Soft scores preserve directional structure:}
We first isolate the geometric effect of the per-plane scoring rule within a single
hash table. Each random hyperplane produces a signed projection of the key, while the
query determines a corresponding score. The next lemma shows that the resulting
correlation with the true query--key similarity is governed by the alignment between
the projected query and the score vector.

\begin{lemma}\label{lem:orth-corr}
% Let $\qb\in\R{d}$ with $\|\qb\|_2=1$, $\kb\sim\mathcal N(0,I_d)$, and
% $\{\hat{\wb}_i\}_{i=1}^P$ be unit vectors with $\hat{\wb}_i^\top\hat{\wb}_j=0$ for $i\neq j$.
% For deterministic scores $s_i=s_i(\qb,\hat{\wb}_i)$, define
Fix a query $\qb\in\R{d}$ with $\|\qb\|_2=1$.
Let $\{\hat{\wb}_i\}_{i=1}^P\subset\R{d}$ be unit vectors with
$\hat{\wb}_i^\top\hat{\wb}_j=0$ for all $i\neq j$, and set
$\Wb := [\hat{\wb}_1,\ldots,\hat{\wb}_P]^\top\in\RR{P}{d}$.
Let $\kb\sim\mathcal N(0,I_d)$ be an independent Gaussian key.
For any deterministic scores $s_i=s_i(\qb,\hat{\wb}_i)\in\mathbb{R}$, define
$X:=\qb^\top\kb$ and $Y:=\sum_{i=1}^P \operatorname{sign}(\hat{\wb}_i^\top\kb)s_i$.
Then $\EE[X]=\EE[Y]=0$, $\EE[X^2]=1$,
$\EE[Y^2]=\sum_{i=1}^P s_i^2$, and
$\EE[XY]=C\sum_{i=1}^P (\qb^\top\hat{\wb}_i)s_i$ with $C=\EE|r|=\sqrt{2/\pi}$,
$r\sim\mathcal N(0,1)$.
Consequently, the correlation between the true similarity signal $X$ and per-table–aggregated hash score $Y$ is given by
$\Gamma:=
% \frac{\EE[XY]}{\sqrt{\EE[X^2]\EE[Y^2]}}
% = C\,\frac{\sum_i (\qb^\top\hat{\wb}_i)s_i}{\sqrt{\sum_i s_i^2}}
C\,\qb^\top \Wb^\ts\hat{\sbb}$,
where $\sbb:=(s_1,\ldots,s_P)^\top$, $\hat{\sbb}:=\sbb/\|\sbb\|_2$.
\end{lemma}
\noindent
Lemma~\ref{lem:orth-corr} gives a general correlation formula for any per-plane
score vector $\sbb$. It shows that the scoring rule affects ranking through the
alignment of the normalized score vector $\hat{\sbb}$ with the projected query
$\Wb\qb$. For hard LSH, the score vector is sign-based and therefore discards magnitude
information in the projected coordinates. In contrast, SOCKET's soft scores vary smoothly with these coordinates and, in the small-signal regime typical of high-dimensional random projections, are approximately linear in $\Wb\qb$. Thus, soft scores preserve more directional information, which helps explain their more stable finite-$L$ rankings. A detailed hard-versus-soft comparison is given in Appendix~\ref{app:add_proofs}. We validate this interpretation empirically by measuring the correlation between
the true signal $\qb^\top\kb$ and the surrogate scores induced by SOCKET and hard
LSH; see Table~\ref{tab:corr_variance_combined}. Using query, key, and value representations
from the final layer of Llama-3.1-8B on SAMSUM and QASPER under a matched
memory budget, SOCKET consistently achieves higher correlation and substantially
lower estimator variance, supporting the view that soft scoring provides a more
stable signal for top-$k$ selection.

\section{Experiments}
\label{sec:exp}

To ensure broad coverage across models supported by all baselines, we evaluate SOCKET on a variety of long-context models of varying scales: Llama-3.1-8B-Instruct~\citep{llama31}, Ministral-3-8B-Instruct-2512~\citep{liu2026ministral3}, Qwen3-8B~\citep{qwen3-8b}, Qwen3-4B-Instruct-2507~\citep{qwen3-8b}, and Qwen3-30B-A3B~\citep{qwen3-8b}. All models support context lengths of up to 128K tokens. We evaluate performance on two complementary benchmarks for long-context understanding: (i) \textbf{LongBench}~\citep{bai2024longbench}, which spans question answering, reasoning, summarization, and code understanding tasks with inputs up to tens of thousands of tokens; and (ii) \textbf{RULER}~\citep{hsieh2024ruler}, a synthetic diagnostic benchmark designed to assess retrieval of sparse, position-sensitive information embedded in very long contexts.
For our experiments, we use the open-source Skylight repository\footnote{\url{https://github.com/skylight-org/sparse-attention-hub}} and implement our method within this framework and test using an NVIDIA H200 GPU. Following the evaluation protocol of vAttention~\citep{desai2025vattention}, we apply dense attention during context processing and sparse attention during question processing and decoding, where the impact of sparsification is most pronounced. Since we adopt this evaluation setup, the results for MagicPIG~\cite{chen2024magicpig} differ from those reported in the original paper. For a fair comparison, we do not retain any dense layers and instead treat all baselines as fully sparse, applying sparsity uniformly across every layer. A detailed discussion about evaluation strategies and how it influences the results for MagicPig can be found in the Appendix~\ref{app:magic_pig}. Consistent with common practice in the sparse attention literature, we include a small number of sink and local window tokens (e.g., 64 or 128 tokens) when computing accuracy in our experiments.
\\
\\
\noindent
\textbf{Baselines:} We compare SOCKET against five representative sparse attention methods: MagicPig~\cite{chen2024magicpig}, HashAttention~\cite{desai2025hashattention}, Quest~\cite{pmlr-v235-tang24l}, Double Sparsity~\cite{yang2024double}, and PQCache~\cite{pqcache2025}. These methods cover a broad spectrum of recent top-$k$–based and sampling-based sparsification approaches. For all baselines, we use the hyperparameter settings recommended by the respective authors to ensure a fair comparison.
% \vspace{-10pt}
\subsection{Effectiveness of SOCKET Across Models}
\vspace{-1.2mm}
Across model families and sparsity regimes, SOCKET consistently matches or outperforms existing baselines. On \textbf{Llama-3.1-8B}, SOCKET outperforms Quest and PQCache on LongBench across all evaluated sparsity levels (Table~\ref{tab:llama-longbench}), achieving $49$ at $10\times$ sparsity, approximately $2$-$3$ points above the strongest baseline, and remaining competitive at $33\times$ sparsity with a score of $45.5$. On RULER-HARD-32K, SOCKET similarly outperforms the baselines for both Llama-3.1-8B and Ministral-3-8B at $5\times$, $10\times$, and $20\times$ sparsity, while remaining competitive with strong data-dependent methods even at $50\times$ sparsity (Tables~\ref{tab:llama-ruler} and~\ref{tab:ministral-ruler}). These results highlight the effectiveness of SOCKET's soft LSH-based scoring under aggressive sparsification. This trend generalizes to \textbf{Qwen3-8B} and \textbf{Ministral-3-8B}, where SOCKET matches or outperforms PQCache and Quest across sparsity levels (Tables~\ref{tab:qwen-longbench} and~\ref{tab:ministral-longbench}). Overall, SOCKET establishes a stronger accuracy--memory Pareto frontier for data-agnostic methods, particularly under light memory budgets (Fig.~\ref{fig:pareto}). Additional results across model scales in the Appendix (Tables~\ref{tab:qwen330-ruler} and~\ref{tab:qwen34-ruler}) further show that SOCKET can retain performance close to dense attention even at extreme sparsity.

\section{Conclusion and Limitations}
\vspace{-0.8em}
In this work, we evaluate SOCKET across three model families -- Llama3~\citep{llama31}, Qwen3~\citep{qwen3-8b}, and Ministral3~\citep{liu2026ministral3} --  demonstrating its effectiveness across diverse model settings. Its generalization to other model families, including hybrid architectures, remains an important direction for future work. Interestingly, for aggregation-heavy tasks (e.g., fwe, govreport, multinews, etc.), SOCKET achieves strong performance using as few as 100 bits, suggesting that highly compact representations may be sufficient for effective KV selection in certain regimes. We leave a systematic investigation of this behavior to future work.
\\
\\
\noindent
SOCKET achieves strong performance with a modest additional memory overhead for hash-table storage, amounting to approximately 15\% beyond the standard KV cache. More broadly, its soft, data-agnostic design makes SOCKET particularly appealing for settings that require stable ranking and efficient selection of relevant entities. By substantially narrowing the performance gap between data-dependent methods (e.g., Quest and PQCache) and data-agnostic approaches (e.g., MagicPig), SOCKET provides a favorable accuracy--memory trade-off, as illustrated in Figure~\ref{fig:pareto}. Taken together, these results position SOCKET as a practical and effective approach to sparse KV-cache selection while highlighting promising opportunities for further reducing its representation and memory requirements.

\section*{Acknowledgments}
The work was primarily supported by Rice Ken Kennedy Institute (K2I) Generative AI Cluster Funding. We are grateful to Professor Yuke Wang for his insights on designing the CUDA scoring kernel.

\setlength{\bibsep}{6pt}
\bibliographystyle{abbrvnat}
{
\bibliography{bibliography.bib}
}

\newpage
\appendix

\section{Appendix}

\subsection{Ablations for SOCKET}\label{sec:abl1}
\label{app:ablation_socket}

We ablate the key \textsc{SOCKET} hyperparameters on five RULER-32K-Hard datasets
(\texttt{nm2}, \texttt{qa1}, \texttt{vt}, \texttt{nm3}, \texttt{qa2})
under the $20\times$ sparsity setting for Llama-3.1-8B-Instruct.
We study the number of hyperplanes $P$, the number of hash tables $L$, and the
soft-hashing temperature $\tau$. The final column reports the average score.

\begin{table}[h]
\centering
\scriptsize
\setlength{\tabcolsep}{3pt}
\renewcommand{\arraystretch}{0.9}

\begin{minipage}[t]{0.32\linewidth}
\vspace{0pt}
\centering
\begin{tabular}{c|cccccc}
\hline
$P$ & nm2 & qa1 & vt & nm3 & qa2 & Avg \\
\hline
4  & 68 & 53 & 79.0 & 1  & 40 & 48.20 \\
5  & 83 & 65 & 78.8 & 15 & 50 & 58.36 \\
6  & 91 & 73 & 85.4 & 52 & 53 & 70.88 \\
7  & 93 & 73 & 87.6 & 84 & 52 & 77.92 \\
8  & 95 & 75 & 92.0 & 91 & 51 & 80.80 \\
9  & 96 & 77 & 93.4 & 91 & 52 & 81.88 \\
10 & 96 & 78 & 96.4 & 94 & 52 & 83.28 \\
\hline
\end{tabular}
\vspace{2pt}
\centerline{\textbf{(a)} Varying $P$; $\tau=0.4$, $L=60$.}
\end{minipage}
\hfill
\begin{minipage}[t]{0.32\linewidth}
\vspace{0pt}
\centering
\begin{tabular}{c|cccccc}
\hline
$L$ & nm2 & qa1 & vt & nm3 & qa2 & Avg \\
\hline
10 & 75 & 61 & 79.2 & 31 & 48 & 58.84 \\
20 & 87 & 71 & 86.6 & 67 & 50 & 72.32 \\
40 & 93 & 78 & 91.2 & 92 & 50 & 80.84 \\
60 & 98 & 76 & 92.8 & 93 & 51 & 82.16 \\
70 & 96 & 76 & 93.0 & 93 & 53 & 82.20 \\
\hline
\end{tabular}
\vspace{2pt}
\centerline{\textbf{(b)} Varying $L$; $\tau=0.5$, $P=10$.}
\end{minipage}
\hfill
\begin{minipage}[t]{0.32\linewidth}
\vspace{0pt}
\centering
\begin{tabular}{c|cccccc}
\hline
$\tau$ & nm2 & qa1 & vt & nm3 & qa2 & Avg \\
\hline
0.1 & 85 & 71 & 78.8 & 71 & 49 & 70.96 \\
0.2 & 94 & 78 & 93.6 & 95 & 50 & 82.12 \\
0.3 & 94 & 82 & 90.0 & 96 & 52 & 82.80 \\
0.4 & 96 & 78 & 96.4 & 94 & 52 & 83.28 \\
0.5 & 98 & 76 & 92.8 & 93 & 51 & 82.16 \\
0.6 & 95 & 75 & 90.6 & 84 & 51 & 79.12 \\
0.7 & 92 & 76 & 85.2 & 67 & 53 & 74.64 \\
0.8 & 89 & 68 & 85.2 & 29 & 51 & 64.40 \\
\hline
\end{tabular}
\vspace{2pt}
\centerline{\textbf{(c)} Varying $\tau$; $P=10$, $L=60$.}
\end{minipage}

\caption{Hyperparameter ablations for \textsc{SOCKET} on RULER-32K-Hard at $20\times$ sparsity using Llama-3.1-8B-Instruct.}
\label{tab:socket_ablation}
\end{table}
\noindent
We observe consistent trends across all three hyperparameters.
Increasing $P$ improves discrimination between buckets and steadily increases accuracy, with gains saturating beyond $P=9$.
Similarly, increasing the number of hash tables $L$ improves recall by reducing approximation error, though the marginal gains diminish after $L=60$.
For the temperature parameter $\tau$, intermediate values ($\tau \in [0.3, 0.5]$) achieve the best performance, balancing the bias--variance trade-off between hard and overly smooth soft hashing. These empirical findings are consistent with Theorem~\ref{thm:main}.
In particular, the finite-$L$ approximation error decreases as $L$ increases, while moderate $\tau$ preserves score discrimination without introducing instability.
Based on this trade-off, we use $P=10$, $L=60$, and $\tau \in [0.3, 0.5]$ in all main experiments.

\subsection{Ablations for Hard LSH Estimator}\label{sec:abl2}
\label{app:ablations_lsh}

In this section, we evaluate Hard LSH under the same memory budget as \textsc{SOCKET} (600 bits per token) on five RULER32K-Hard datasets (\texttt{nm2}, \texttt{qa1}, \texttt{vt}, \texttt{nm3}, and \texttt{qa2}) at $20\times$ sparsity using Llama-3.1-8B-Instruct.
The corresponding \textsc{SOCKET} results under the same budget are provided in Section~\ref{app:ablation_socket}.

\begin{table}[h]
\centering
\scriptsize
\setlength{\tabcolsep}{3pt}
\renewcommand{\arraystretch}{0.9}

% -------- Row 1 --------
\begin{minipage}[t]{0.48\linewidth}
\vspace{0pt}
\centering
\begin{tabular}{c|cccccc}
\hline
$P$ & nm2 & qa1 & vt & nm3 & qa2 & Avg \\
\hline
1 & 40 & 46 & 71.2 & 3  & 44 & 40.84 \\
2 & 70 & 63 & 79.4 & 30 & 47 & 57.88 \\
3 & 74 & 63 & 80.6 & 22 & 49 & 57.72 \\
4 & 62 & 60 & 72.4 & 20 & 49 & 52.68 \\
5 & 43 & 58 & 63.2 & 5  & 48 & 43.44 \\
\hline
\end{tabular}
\vspace{2pt}
\centerline{\textbf{(a)} Varying $P$ ($L=60$).}
\end{minipage}
\hfill
\begin{minipage}[t]{0.48\linewidth}
\vspace{0pt}
\centering
\begin{tabular}{cc|cccccc}
\hline
$L$ & Bits & nm2 & qa1 & vt & nm3 & qa2 & Avg \\
\hline
70  & 140  & 76 & 69 & 81.4 & 33 & 47 & 61.28 \\
100 & 200  & 85 & 71 & 85.6 & 66 & 49 & 71.32 \\
150 & 300  & 93 & 73 & 89 & 80 & 51 & 77.20 \\
200 & 400  & 94 & 76 & 91 & 88 & 50 & 79.80 \\
250 & 500  & 94 & 76 & 94 & 93 & 51 & 81.60 \\
300 & 600  & 97 & 81 & 95 & 94 & 53 & 84 \\
\hline
\end{tabular}
\vspace{2pt}
\centerline{\textbf{(b)} Varying $L$ (fixed $P=2$).}
\end{minipage}

\vspace{6pt} % space between rows

% -------- Row 2 --------
\begin{minipage}[t]{0.6\linewidth}
\vspace{0pt}
\centering
\begin{tabular}{cc|cccccc}
\hline
$L$ & Bits & nm2 & qa1 & vt & nm3 & qa2 & Avg \\
\hline
350 & 700  & 97 & 81 & 95.2 & 94 & 53 & 84.04 \\
400 & 800  & 97 & 81 & 95.8 & 93 & 52 & 83.76 \\
450 & 900  & 96 & 79 & 95.8 & 91 & 50 & 82.36 \\
500 & 1000 & 98 & 82 & 96.0 & 95 & 53 & 84.8 \\
\hline
\end{tabular}
\vspace{2pt}
\centerline{\textbf{(c)} Larger memory budget.}
\end{minipage}

\caption{Hard LSH ablations on RULER32K-Hard at $20\times$ sparsity (Llama-3.1-8B-Instruct).}
\label{tab:hardlsh_ablation}
\end{table}
\noindent
Hard LSH shows strong sensitivity to the hyperparameter $P$.
The best performance is achieved at $P=2$, while larger values substantially degrade accuracy due to increasingly sparse and brittle bucket assignments.
Unlike \textsc{SOCKET}, which leverages soft bucket probabilities, hard LSH relies on discrete collisions and therefore requires substantially larger $L$ to recover comparable retrieval quality. Under the same memory budget of 600 bits/token, hard LSH reaches an average score of 84, still below \textsc{SOCKET}'s 85.08. Even when the memory budget is increased to 1000 bits/token, Hard LSH reaches 84.80, which remains lower than \textsc{SOCKET} under significantly smaller memory usage.

\subsection{Comparison with MagicPIG}
\label{app:magic_pig}

The variation in the reported results relative to the original MagicPIG~\cite{chen2024magicpig} primarily arises from differences in the evaluation protocol.
As discussed in vAttention~\cite{desai2025vattention} (Table~10), two evaluation setups are commonly used for sparse attention methods.

\paragraph{Setup A:} In this setup, the full context and question are first processed using dense attention, and sparsity is applied only during the decoding stage.
This is the setup used in MagicPIG~\cite{chen2024magicpig}.

\paragraph{Setup B:} In this setup, sparsity is applied during both question processing and decoding. This is the setup used in our evaluation. Compared to Setup A, this setting is substantially more challenging, since the sparse method must not only support efficient decoding but also retrieve relevant information while processing the question itself.
\\
\\
\noindent
In addition, MagicPig retains (0, 16) dense layers as a fallback mechanism. To ensure a fair comparison under a fully sparse setting, we remove all dense layers and evaluate MagicPig as a completely sparse method in Table~\ref{tab:lama31-ruler}, applying sparsity uniformly across every layer. We additionally evaluate a hybrid dense--sparse variant that retains the 0th and 16th layers as dense, consistent with the original MagicPig design, while making the remaining layers correspondingly sparser to preserve a comparable overall sparsity ratio.
The results are shown in Table~\ref{tab:magicpig_appendix}.

\begin{table}[h]
\centering
\small
\setlength{\tabcolsep}{4pt}
\caption{Comparison with MagicPIG under different evaluation settings.}
\label{tab:magicpig_appendix}
\vspace{2pt}
\begin{tabular}{lccccccc}
\hline
Method & Sparsity & nm2 & nm3 & vt & qa1 & qa2 & Avg \\
\hline
MagicPIG (0,16 dense) & $5\times$  
& 77 & 71 & 96.8 & 69.6 & 46 & 72.1 \\

MagicPIG (0,16 dense) & $10\times$ 
& 35 & 16 & 95.8 & 53 & 46 & 49.2 \\

MagicPIG (0,16 dense) & $50\times$ 
& 0.6 & 0 & 81.8 & 24 & 32 & 27.7 \\
\hline

MagicPIG (fully sparse) & $5\times$  
& 49 & 23 & 93.3 & 63 & 47 & 55.1 \\

MagicPIG (fully sparse) & $10\times$ 
& 10.3 & 0 & 93 & 47 & 42 & 38.5 \\

MagicPIG (fully sparse) & $50\times$ 
& 0.3 & 0 & 70.6 & 21 & 28 & 24.0 \\
\hline

\textsc{SOCKET} & $5\times$  
& 97 & 100 & 95.2 & 81 & 53 & \textbf{85.2} \\

\textsc{SOCKET} & $10\times$ 
& 96 & 99.6 & 94.2 & 80 & 52 & \textbf{84.4} \\

\textsc{SOCKET} & $50\times$ 
& 80 & 73 & 83.8 & 77 & 50 & \textbf{72.8} \\
\hline
\end{tabular}
\end{table}
\noindent
As expected, the hybrid dense--sparse MagicPig configuration performs better than its fully sparse counterpart. However, even under this more favorable setting, MagicPIG remains consistently below \textsc{SOCKET} across all sparsity levels.

\newpage
\subsection{Additional Experiments and Error Bars}\label{sec:addl}

\begin{table*}[t]
\centering

% =========================================================
% LEFT: Qwen3-30B-A3B
% =========================================================
\begin{minipage}[t]{0.48\textwidth}
\vspace{0pt}
\centering
\scriptsize
\setlength{\tabcolsep}{3.2pt}
\renewcommand{\arraystretch}{1.0}

\begin{tabular}{@{}crrrrrrr@{}}
\toprule
Spr & vt & qa1 & qa2 & fwe & nm2 & nm3 & Avg \\
\midrule

Dense
& 99.8 & 88 & 59 & 100 & 99.3 & 100 & 91.0 \\

\midrule

$5\times$
& 100 & 88 & 60 & 98.6 & 100 & 100 & \textbf{91.1} \\

$10\times$
& 100 & 87 & 62 & 97.6 & 100 & 100 & \textbf{91.1} \\

$20\times$
& 100 & 87 & 60 & 97.3 & 100 & 100 & 90.7 \\

$50\times$
& 100 & 85 & 57 & 95.6 & 100 & 100 & 89.6 \\

\bottomrule
\end{tabular}

\captionof{table}{
RULER-HARD-32K performance using Qwen3-30B-A3B.
Sparse results use SOCKET ($P=10$, $L=60$, $\tau=0.3$).
}
\label{tab:qwen330-ruler}

\end{minipage}
\hfill
% =========================================================
% RIGHT: Qwen3-4B
% =========================================================
\begin{minipage}[t]{0.48\textwidth}
\vspace{0pt}
\centering
\scriptsize
\setlength{\tabcolsep}{3.2pt}
\renewcommand{\arraystretch}{1.0}

\begin{tabular}{@{}crrrrrrr@{}}
\toprule
Spr & vt & qa1 & qa2 & fwe & nm2 & nm3 & Avg \\
\midrule

Dense
& 100 & 79 & 60 & 95 & 99 & 99 & \textbf{88.7} \\

\midrule

$5\times$
& 100 & 78 & 60 & 94 & 99 & 99 & 88.3 \\

$10\times$
& 100 & 78 & 53 & 92.3 & 99 & 98 & 86.7 \\

$20\times$
& 100 & 76 & 52 & 92 & 99 & 99 & 86.3 \\

$50\times$
& 100 & 68 & 52 & 91 & 92 & 76 & 79.8 \\

\bottomrule
\end{tabular}

\captionof{table}{
RULER-HARD-32K performance using Qwen3-4B-Instruct-2507.
Sparse results use SOCKET ($P=10$, $L=60$, $\tau=0.3$).
}
\label{tab:qwen34-ruler}

\end{minipage}

\end{table*}

\begin{table*}[t]
\centering

% =========================================================
% LEFT: Qwen3-30B-A3B
% =========================================================
\begin{minipage}[t]{0.48\textwidth}
\vspace{0pt}
\centering
\scriptsize
\setlength{\tabcolsep}{3.5pt}
\renewcommand{\arraystretch}{1.0}

\begin{tabular}{@{}crrrrrr@{}}
\toprule
Spr & vt & qa1 & qa2 & fwe & nm2 & nm3 \\
\midrule

$5\times$
& 0 & 0.5 & 0 & 0 & 0.5 & 0 \\

$10\times$
& 0 & 2.0 & 0.5 & 0.5 & 0 & 0 \\

$20\times$
& 0 & 1.5 & 0 & 0.2 & 1.0 & 0 \\

$50\times$
& 0 & 1.1 & 1.2 & 1.0 & 0.5 & 2.0 \\

\bottomrule
\end{tabular}

\captionof{table}{
Standard deviations for SOCKET ($P=10$, $L=60$) using Qwen3-30B-A3B.
}
\label{tab:socket_qwen3_30b_std}

\end{minipage}
\hfill
% =========================================================
% RIGHT: Qwen3-4B
% =========================================================
\begin{minipage}[t]{0.48\textwidth}
\vspace{0pt}
\centering
\scriptsize
\setlength{\tabcolsep}{3.5pt}
\renewcommand{\arraystretch}{1.0}

\begin{tabular}{@{}crrrrrr@{}}
\toprule
Spr & vt & qa1 & qa2 & fwe & nm2 & nm3 \\
\midrule

$5\times$
& 0 & 0.5 & 0 & 0 & 0.5 & 0 \\

$10\times$
& 0 & 2.0 & 0.5 & 0.5 & 0 & 0 \\

$20\times$
& 0 & 1.5 & 0 & 0.2 & 1.0 & 0 \\

$50\times$
& 0 & 1.1 & 1.2 & 1.0 & 0.5 & 2.0 \\

\bottomrule
\end{tabular}

\captionof{table}{
Standard deviations for SOCKET ($P=10$, $L=60$) using Qwen3-4B-Instruct-2507.
}
\label{tab:socket_qwen3_4b_std}

\end{minipage}

\end{table*}

\begin{table*}[t]
\centering

% =========================================================
% LEFT: Llama-3.1-8B
% =========================================================
\begin{minipage}[t]{0.49\textwidth}
\vspace{0pt}
\centering
\scriptsize
\setlength{\tabcolsep}{2.8pt}
\renewcommand{\arraystretch}{0.95}

\begin{tabular}{@{}lrrrrrr@{}}
\toprule
Method & nm2 & nm3 & vt & fwe & qa1 & qa2 \\
\midrule

\multicolumn{7}{c}{\textit{5$\times$ sparsity}} \\
\midrule

PQCache
& 0 & 0 & 0 & 0.1 & 0 & 0 \\

MagicPig
& 3.0 & 2.5 & 0.5 & 1.6 & 1.0 & 1.7 \\

SOCKET
& 0.3 & 0 & 0.3 & 0 & 1.5 & 0.5 \\

\midrule
\multicolumn{7}{c}{\textit{10$\times$ sparsity}} \\
\midrule

PQCache 
& 0.5 & 0.5 & 0.1 & 0.1 & 0.5 & 0.5 \\

MagicPig 
& 3.2 & 0 & 0.6 & 1.5 & 2.2 & 1.7 \\

SOCKET
& 0.5 & 0.5 & 0.3 & 0.5 & 2.6 & 0.5 \\

\midrule
\multicolumn{7}{c}{\textit{20$\times$ sparsity}} \\
\midrule

PQCache 
& 2.0 & 1.0 & 0.2 & 0.3 & 0 & 1.0 \\

MagicPig 
& 1.0 & 0 & 2.9 & 2.3 & 1.5 & 2.8 \\

SOCKET
& 1.7 & 1.5 & 1.3 & 1.3 & 2.5 & 0.5 \\

\midrule
\multicolumn{7}{c}{\textit{50$\times$ sparsity}} \\
\midrule

PQCache 
& 2.0 & 0.5 & 0.7 & 1.1 & 2.0 & 1.7 \\

MagicPig
& 0.5 & 0 & 2.5 & 0.8 & 1.1 & 2.0 \\

SOCKET
& 3.0 & 3.4 & 1.7 & 1.5 & 1.1 & 1.0 \\

\bottomrule
\end{tabular}

\captionof{table}{
Standard deviations across 3 different seeds for Table~\ref{tab:llama-ruler}.
}
\label{tab:llama31-ruler-std}

\end{minipage}
\hfill
% =========================================================
% RIGHT: Ministral-3-8B
% =========================================================
\begin{minipage}[t]{0.49\textwidth}
\vspace{0pt}
\centering
\scriptsize
\setlength{\tabcolsep}{2.8pt}
\renewcommand{\arraystretch}{0.95}

\begin{tabular}{@{}lrrrrrr@{}}
\toprule
Method & nm2 & nm3 & vt & fwe & qa1 & qa2 \\
\midrule

\multicolumn{7}{c}{\textit{5$\times$ sparsity}} \\
\midrule

PQCache 
& 0 & 0 & 0 & 0.2 & 0.5 & 0 \\

MagicPig 
& 3.0 & 0.5 & 0 & 1.3 & 1.5 & 1.5 \\

SOCKET 
& 0.5 & 0 & 0 & 0.4 & 1.0 & 1.5 \\

\midrule
\multicolumn{7}{c}{\textit{10$\times$ sparsity}} \\
\midrule

PQCache
& 0 & 0 & 0 & 0.2 & 1.2 & 0.5 \\

MagicPig 
& 3.0 & 1.0 & 0 & 1.7 & 3.0 & 2.5 \\

SOCKET 
& 0 & 0 & 0 & 1.5 & 1.5 & 0.5 \\

\midrule
\multicolumn{7}{c}{\textit{20$\times$ sparsity}} \\
\midrule
PQCache 
& 0.5 & 0.5 & 0 & 0.2 & 1.2 & 0 \\

MagicPig 
& 1.7 & 0 & 0.4 & 3.0 & 2.5 & 3.1 \\

SOCKET
& 0 & 2.3 & 0 & 0.5 & 1.5 & 0.5 \\

\midrule
\multicolumn{7}{c}{\textit{50$\times$ sparsity}} \\
\midrule

PQCache 
& 1.0 & 1.2 & 0 & 1.0 & 0.6 & 1.6 \\

MagicPig 
& 0 & 0 & 2.0 & 0.5 & 2.3 & 2.5 \\

SOCKET
& 0 & 2.3 & 0 & 1.5 & 2 & 1.5 \\

\bottomrule
\end{tabular}

\captionof{table}{
Standard deviations across 3 different seeds for Table~\ref{tab:ministral-ruler}.
}
\label{tab:ministral-ruler-std}

\end{minipage}

\end{table*}

% =========================================================
% APPENDIX TABLE: STANDARD DEVIATIONS ONLY
% =========================================================
\begin{table*}[t]
\centering
\caption{Standard deviations across 3 different seeds for Table~\ref{tab:llama-longbench}.}
\label{tab:llama31-long-std}
\resizebox{\textwidth}{!}{
\begin{tabular}{lccccccccccccccc}
\toprule
Method & Sparsity & NQA & QAS & MFQA & HPQA & WIKI & MUS & GOV & QMSUM & MNews & LCC & Trivia & SamSUM & Retrieval & Repo \\
\midrule

PQCache & 10$\times$
& 0.5 & 0.6 & 0.5 & 0.1 & 0.2 & 0.7 & 0.2 & 0.1 & 0.3 & 0.5 & 1.1 & 0.5 & 0.0 & 0.6 \\

MagicPig & 10$\times$
& 1.1 & 0.8 & 0.6 & 0.4 & 2.4 & 2.4 & 0.6 & 0.1 & 0.2 & 1.3 & 1.8 & 2.1 & 3.0 & 2.0 \\

SOCKET & 10$\times$
& 0.7 & 0.6 & 0.2 & 0.4 & 1.5 & 1.0 & 0.2 & 0.4 & 0.1 & 0.2 & 0.4 & 0.3 & 0.0 & 0.3 \\

\midrule

PQCache & 33$\times$
& 0.8 & 1.2 & 0.9 & 1.3 & 1.5 & 0.4 & 0.1 & 0.2 & 0.1 & 0.9 & 0.8 & 0.6 & 0.2 & 0.2 \\

MagicPig & 33$\times$
& 0.8 & 1.4 & 1.4 & 3.0 & 1.4 & 1.6 & 0.2 & 0.2 & 0.1 & 0.9 & 3.0 & 2.1 & 3.0 & 1.9 \\

SOCKET & 33$\times$
& 0.3 & 0.8 & 1.0 & 1.8 & 1.9 & 0.9 & 0.4 & 0.1 & 0.6 & 0.7 & 0.7 & 1.2 & 0.7 & 0.8 \\

\bottomrule
\end{tabular}
}
\end{table*}

\begin{table*}[t]
\centering
\caption{Standard deviations across 3 different seeds for Table~\ref{tab:qwen-longbench}.}
\label{tab:qwen-long-std}
\resizebox{\textwidth}{!}{
\begin{tabular}{lccccccccccccccc}
\toprule
Method & Sparsity & NQA & QAS & MFQA & HPQA & WIKI & MUS & GOV & QMSUM & MNews & LCC & Trivia & SamSUM & Retrieval & Repo \\
\midrule

PQCache (pq=8, gr=4)
& 10$\times$
& 0.6 & 0.3 & 0.1 & 0.3 & 0.6 & 0.5 & 0.1 & 0.0 & 0.0 & 0.0 & 0.3 & 0.1 & 0.5 & 0.2 \\

MagicPig (P=8, L=120)
& 10$\times$
& 0.1 & 1.4 & 0.6 & 0.9 & 2.0 & 1.6 & 0.0 & 0.3 & 0.0 & 0.1 & 1.9 & 0.5 & 1.4 & 1.3 \\

SOCKET (P=8, L=50)
& 10$\times$
& 0.5 & 0.5 & 0.4 & 0.7 & 0.2 & 0.3 & 0.0 & 0.1 & 0.2 & 0.7 & 0.2 & 0.2 & 0.5 & 0.8 \\

\midrule

PQCache (pq=8, gr=4)
& 33$\times$
& 0.3 & 0.8 & 0.6 & 0.9 & 0.8 & 2.2 & 0.1 & 0.3 & 0.1 & 0.4 & 0.6 & 0.3 & 0.6 & 0.3 \\

MagicPig (P=10, L=200)
& 33$\times$
& 1.3 & 0.9 & 2.0 & 2.0 & 2.3 & 2.8 & 0.3 & 0.2 & 0.4 & 0.6 & 3.0 & 0.2 & 0.7 & 1.4 \\

SOCKET (P=8, L=50)
& 33$\times$
& 0.9 & 1.5 & 0.8 & 1.0 & 2.0 & 1.8 & 0.1 & 0.1 & 0.1 & 0.8 & 0.5 & 0.5 & 0.7 & 0.3 \\

\bottomrule
\end{tabular}
}
\end{table*}

\begin{table*}[t]
\centering
\caption{Standard deviations across 3 different seeds for Table~\ref{tab:ministral-longbench}.}
\label{tab:ministral-long-std}
\resizebox{\textwidth}{!}{
\begin{tabular}{lccccccccccccccc}
\toprule
Method & Sparsity & NQA & QAS & MFQA & HPQA & WIKI & MUS & GOV & QMSUM & MNews & LCC & Trivia & SamSUM & Retrieval & Repo \\
\midrule

PQCache (pq=8, gr=4)
& 10$\times$
& 0.4 & 0.4 & 0.3 & 0.4 & 0.9 & 0.4 & 0.1 & 0 & 0.2 & 0.2 & 0.2 & 0.2 & 0.1 & 0.4 \\

MagicPig (P=8, L=120)
& 10$\times$
& 1.3 & 1.8 & 0.9 & 1.3 & 2.0 & 2.0 & 0.3 & 0.2 & 0.1 & 2.0 & 0.7 & 1.1 & 0.2 & 2.0 \\

SOCKET (P=8, L=50)
& 10$\times$
& 0.7 & 0.4 & 0.6 & 0.6 & 1.2 & 0.7 & 0.1 & 0.1 & 0.2 & 1.0 & 0.3 & 0.3 & 0.1 & 1.2 \\

\midrule

PQCache (pq=8, gr=4)
& 33$\times$
& 0.2 & 0.3 & 0.8 & 1.1 & 1.2 & 0.4 & 0 & 0.2 & 0.3 & 0.2 & 0.6 & 0.7 & 0 & 0.3 \\

MagicPig (P=10, L=200)
& 33$\times$
& 1.4 & 0.4 & 1.0 & 2.0 & 2.0 & 0.3 & 0.4 & 0.2 & 0.3 & 1.5 & 0.3 & 0.7 & 1.7 & 0.7 \\

SOCKET (P=8, L=50)
& 33$\times$
& 0.6 & 1.2 & 0.5 & 1.3 & 1.6 & 1.3 & 0.2 & 0.2 & 0.1 & 1.2 & 0.7 & 0.6 & 0.1 & 2.4 \\

\bottomrule
\end{tabular}
}
\end{table*}

\begin{table}[H]
\centering
\caption{Hyperparameter settings used across datasets and models.}
\label{tab:hyperparams}
\begin{tabular}{ccccl}
\toprule
$P$ & $L$ & $\tau$ & Dataset & Model \\
\midrule
8  & 60 & 0.3-0.7 &LongBench      & Llama-3.1-8B-Instruct \\
8  & 60 & 0.3-0.7 & LongBench      & Llama-3.2-1B-Instruct \\
8  & 60 & 0.3-0.7 & LongBench      & Qwen3-8B \\
10 & 60 & 0.3-0.7 & RULER-16K       & Llama-3.1-8B-Instruct \\
10 & 60 & 0.3-0.7 & RULER-32K       & Llama-3.1-8B-Instruct \\
10 & 60 & 0.3-0.7 & RULER-32K       & Qwen3-4B-Instruct-2507 \\
10 & 60 & 0.3-0.7 & RULER-32K       & Qwen3-30B-A3B \\
\bottomrule
\end{tabular}
\end{table}

% \newpage
\subsection{Definitions of metrics used in fig.~\ref{fig:ranking_metrics}}
\label{sec:metric_def}
\textbf{Normalized Discounted Cumulative Gain (NDCG):} NDCG measures the quality of a ranked list by accounting for both item relevance and ranking position, assigning higher weight to relevant items appearing earlier.
Given a ranked list of length $k$ with relevance scores $\{r_i\}_{i=1}^k$, the discounted cumulative gain is defined as
\[
\mathrm{DCG} = \sum_{i=1}^{k} \frac{2^{r_i} - 1}{\log_2(i + 1)}.
\]
Let $\mathrm{IDCG}$ denote the DCG obtained by sorting items in decreasing order of relevance. The normalized score is then
\[
\mathrm{NDCG} = \frac{\mathrm{DCG}}{\mathrm{IDCG}}.
\]
\noindent
\textbf{Precision:} Precision measures the fraction of retrieved items that are relevant.
Let $S_k$ denote the set of top-$k$ retrieved items and $R$ the set of relevant items. Precision is defined as
\[
\mathrm{Precision} = \frac{|S_k \cap R|}{k}.
\]

\textbf{Jaccard Similarity:} Jaccard similarity measures the overlap between two sets.
Given two sets $A$ and $B$, it is defined as
\[
\mathrm{Jaccard}(A, B) = \frac{|A \cap B|}{|A \cup B|}.
\]

\section{Proof of Theorem~\ref{thm:main}}\label{thm:proof}
\label{app:proof_b}
% \begin{table}[h]
% \centering
% \caption{Hyperparameter settings used across datasets and models.}
% \label{tab:hyperparams}
% \begin{tabular}{ccccl}
% \toprule
% $P$ & $L$ & $\tau$ & Dataset & Model \\
% \midrule
% 8  & 60 & 0.3-0.7 &LongBench      & Llama-3.1-8B-Instruct \\
% 8  & 60 & 0.3-0.7 & LongBench      & Llama-3.2-1B-Instruct \\
% 8  & 60 & 0.3-0.7 & LongBench      & Qwen3-8B \\
% 10 & 60 & 0.3-0.7 & RULER-16K       & Llama-3.1-8B-Instruct \\
% 10 & 60 & 0.3-0.7 & RULER-32K       & Llama-3.1-8B-Instruct \\
% \bottomrule
% \end{tabular}
% \end{table}

\subsection{Additional notes on Theorem~\ref{thm:main}}\label{sec:notations}
Theorem~\ref{thm:main} provides an end-to-end error decomposition for the proposed
soft-count attention estimator $\Tb(\qb)$. The bound separates the total error into
three components: (i) sampling variance arising from the value-aware sampling step,
(ii) finite-table approximation error due to using a limited number of hash tables,
and (iii) a bias induced by soft bucketization, quantified by
$\varepsilon_\tau=\EE\!\left[1-p_\tau^{(\ell)}(b_q\mid \qb)\right]$.
This decomposition makes explicit how the algorithmic parameters jointly control
the different sources of error.
%
% Due to space constraints, the complete proof of Theorem~\ref{thm:main} is deferred to Appendix~\ref{thm:proof}. 
To provide intuition, the analysis is organized around the
following triangle inequality:
\begin{align}
\|\mathbf{T}(\qb)-\yb^*(\qb)\|_2
\;\le\;
\|\mathbf{T}(\qb)-\yb_{\tau,L}(\qb)\|_2
+\|\yb_{\tau,L}(\qb)-\yb_\tau(\qb)\|_2
+\|\yb_\tau(\qb)-\yb^*(\qb)\|_2 ,
\label{eq:triangle_b}
\end{align}
which isolates the contributions of sampling variance, finite-table effects, and
soft-bucketization bias, respectively. Sections~\ref{sec:int} and \ref{proof_final} revolve around bounding each of them separately.

\paragraph{Role of $M$ and $L$.}
The terms $M^{-1/2}$ and $L^{-1/2}$ capture two independent variance sources.
The $M^{-1/2}$ term arises from the value-aware sampling step and reflects the
Monte Carlo error incurred when approximating the soft attention output using
$M$ sampled values.
The $L^{-1/2}$ term corresponds to the finite-table approximation of the soft
collision kernel and quantifies the concentration of the $L$-table estimator
around its population limit.
Both terms decay at the standard parametric rate and are independent of the
sequence length $N$, highlighting the scalability of the method. They can be controlled separately depending on computational and
memory constraints.

\paragraph{Role of the temperature $\tau$.}
The term $\varepsilon_\tau(\qb)$ in Theorem~\ref{thm:main} captures the approximation bias induced by soft bucketization.
Fix a table $\ell$ and define the hard SRP bucket of the query by
\[
b_q^{(\ell)} := h^{(\ell)}(\qb)=\operatorname{sign}(\Wb^{(\ell)}\qb)\in\{\pm1\}^P
% \quad\text{(equivalently, an index in $[R]$ with $R=2^P$)}.
\]
Recall that Algorithm~\ref{alg:soft-bucket-probs} defines
\[
p_\tau^{(\ell)}(r\mid\qb)
=
\frac{\exp\!\big(\ub^{(\ell)}(\qb)^\top \cbb_r/\tau\big)}
{\sum_{r'=1}^{R}\exp\!\big(\ub^{(\ell)}(\qb)^\top \cbb_{r'}/\tau\big)},
% \qquad
% \ub^{(\ell)}(\qb)=\frac{1}{\sqrt d}\tanh(\Wb^{(\ell)}\qb)\in\R^P.
\]
We define the peaking (non-dominant mass) quantity
\[
\varepsilon_\tau(\qb)
:=
\mathbb{E}_{\Wb^{(\ell)}}\!\left[\,1-p_\tau^{(\ell)}\!\big(b_q^{(\ell)}\mid\qb\big)\right],
\]
which measures how much probability mass the soft bucket distribution assigns to buckets other than the hard query bucket.
To see how $\tau$ controls this bias, write the bucket logits as
$x_r := \ub^{(\ell)}(\qb)^\top \cbb_r$ for $r\in[R]$, so that
$p_\tau^{(\ell)}(r\mid\qb)=\exp(x_r/\tau)/\sum_{r'}\exp(x_{r'}/\tau)$.
Let $r^\star:=\arg\max_{r\in[R]} x_r$. Since $\tanh(\cdot)$ is strictly increasing coordinatewise,
we have $r^\star=b_q^{(\ell)}$, and factoring out the maximum logit yields
\begin{flalign}
p_\tau^{(\ell)}(b_q\mid\qb)=\frac{\exp\!\big(x_{b_\qb}/\tau)}{\sum_{r=1}^R \exp\!\big(x_{r}/\tau)}
=\frac{1}{1+\sum_{r\neq b_q}\exp\!\big((x_r-x_{b_q})/\tau\big)}\nonumber
\end{flalign}
Because $x_r-x_{b_q}<0$ for all $r\neq b_q$ and $R=2^P$ is finite for fixed $P$,
we have $\exp((x_r-x_{b_q})/\tau)\to 0$ as $\tau\to 0$, hence
$p_\tau^{(\ell)}(b_q\mid\qb)\to 1$ and therefore $\varepsilon_\tau(\qb)\to 0$ as $\tau\to 0$ (for fixed $P$).
In the opposite extreme $\tau\to\infty$, we have $\exp(x_r/\tau)=1+\Ocal(1/\tau)$ and thus
$p_\tau^{(\ell)}(r\mid\qb)\to 1/R$ for all $r$, i.e., the bucket distribution uniformizes and
$\varepsilon_\tau(\qb)\to 1-1/R$.
Consequently, $\tau$ interpolates between hard bucketing ($\tau\to 0$) and uniform bucket weights ($\tau\to\infty$),
trading approximation bias against smoother, more stable similarity scores at finite $\tau$.

\paragraph{Error-bound without sampling.}
If the value-aware sampling step is omitted, the estimator reduces to the finite-table
soft-count attention output $\yb_{\tau,L}(\qb)$.
In this setting, the randomness arises solely from the $L$ random hash tables, and the
sampling variance term in Theorem~\ref{thm:main} disappears.
Consequently, under the same assumptions and with probability at least $1-\delta$ over the
hash-table randomness, we have
\[
\|\yb_{\tau,L}(\qb)-\yb^*(\qb)\|_2
=
\widetilde{\mathcal O}\!\left(
\frac{1}{\sqrt{L}}
+\varepsilon_\tau(\qb)
\right)\|\Vb\|_2.
\]
The $L^{-1/2}$ term captures the concentration of the finite-table soft-count estimator
around its population limit $\yb_\tau(\qb)$, while $\varepsilon_\tau(\qb)$ quantifies the bias
introduced by soft bucketization relative to angular attention.
Thus, in the absence of sampling, soft collision aggregation provides a statistically
controlled proxy for angular attention, with accuracy governed by the number of hash tables
and the temperature parameter.

\subsection{Intermediate Lemmas}\label{sec:int}
Fix a query $\qb \in \mathbb{R}^d$, keys $\kb_1,\dots \kb_N \in\mathbb{R}^d$ and values $\mathbf{v}_1,\dots \mathbf{v}_N\in \mathbb{R}^{d}$. Define the angular kernel weights $w_j=\left(1-\frac{1}{\pi} \cos ^{-1}\left(\nicefrac{\qb^{\top} \kb_j}{\|\qb\|\left\|\kb_j\right\|}\right)\right)^P \in[0,1]$ and $Z:=\sum_{j=1}^N w_j$. With this, define the angular attention distribution and output:

$$
a_j:=\frac{w_j}{Z}, \quad \yb^*(\qb):=\sum_{j=1}^N a_j \vb_j .
$$
\noindent
\textbf{\emph{Soft-count} distribution}: For each table $\ell \in[L]$, draw hyperplanes $\Wb^{(\ell)} \in \mathbb{R}^{P \times d}$ with i.i.d. $\mathcal{N}(0,1)$ rows. Let corners $\mathbf{c}_r \in\{ \pm 1\}^P$ for $r \in[R], R=2^P$. Define

$$
\mathbf{u}^{(\ell)}(\qb):=\frac{1}{\sqrt{d}} \tanh \left(\Wb^{(\ell)} \qb\right) \in \mathbb{R}^P, \quad p_\tau^{(\ell)}(r \mid \qb):=\frac{\exp \left(\mathbf{u}^{(\ell)}(\qb)^\top \mathbf{c}_r / \tau\right)}{\sum_{r^{\prime}=1}^R \exp \left(\mathbf{u}^{(\ell)}(\qb)^\top \mathbf{c}_{r^{\prime}}/ \tau\right)} .
$$

\noindent
Let the $\kb_j$'s (hard) bucket id be
$
b^{(\ell)}_j \in[R] \quad \text {(equivalently the sign pattern of } \Wb^{(\ell)} \kb_j \text { ). }
$ Define the per-table soft score and the $L$-table soft-count:

$$
s_j^{(\ell)}(\qb):=p_\tau^{(\ell)}\left(b^{(\ell)}_j \mid \qb\right) \in[0,1], \quad \widetilde{w}_j:=\frac{1}{L} \sum_{\ell=1}^L s_j^{(\ell)}(\qb), \quad \widetilde{Z}:=\sum_{j=1}^N \widetilde{w}_j, \quad \widetilde{a}_j:=\frac{\widetilde{w}_j}{\widetilde{Z}} .
$$

\noindent
Now, we define the ``soft attention" output:
$
\yb_{\tau, L}(\qb):=\sum_{j=1}^N \widetilde{a}_j \vb_j=\frac{\widetilde{n}(\qb)}{\widetilde{Z}}\in\mathbb{R}^d$. Similarly, we define the population (single-table) soft-count weight
$
w_{\tau, j}:=\mathbb{E}\left[s_j^{(1)}(\qb)\right], \quad Z_\tau:=\sum_{j=1}^N w_{\tau, j}, \quad a_{\tau, j}:=\frac{w_{\tau, j}}{Z_\tau}$.
Accordingly, the soft-attention output of the population is given by
$
\yb_\tau(\qb):=\sum_{j=1}^N a_{\tau, j} \vb_j=\frac{\mathbf{n}_{\tau}(\qb)}{Z_\tau}
$,
where the expectations are over the randomness in the table, i.e. $\Wb^{(\ell)}$ and $\mathbf{n}_{\tau}(\qb)=\sum_{j=1}^N w_{\tau, j} \vb_j \in \mathbb{R}^d$.
\noindent
Note that Assumption~\ref{as2} directly implies

\begin{flalign}
 Z^{(\ell)}(\qb):=&~\sum_{j=1}^N s_j^{(\ell)}(\qb)=\sum_{j=1}^N p_\tau^{(\ell)}\left(b^{(\ell)}_j \mid \qb\right)=\sum_{r=1}^R \sum_{b_j^{(\ell)}=r} p_\tau^{(\ell)}(r \mid \qb)\nonumber\\
 =&~\sum_{r=1}^R p_\tau^{(\ell)}(r \mid \qb)(\sum_{b_j^{(\ell)}=r} 1)\le B\sum_{r=1}^R p_\tau^{(\ell)}(r \mid \qb)=B.\label{eq:A2}
\end{flalign}

% \subsection{Results}

%%%%%%%%%%%%%%%%%%%%%%%%%%%%%%%%%%%%%%%%%%%%%%%%

\begin{lemma}[Concentration of $\widetilde{Z}$ and $\widetilde{\mathbf{n}}(\qb)$]\label{lem:zhat}
If $L\ge\frac{2B^2\log (4/\delta_\mathrm{L})}{Z_{\tau, \min }^2}$, then with probability at least $1-\delta_\mathrm{L}$,
\begin{flalign}
\left|\widetilde{Z}-Z_\tau\right| \leq B \sqrt{\frac{\log (4 / \delta_\mathrm{L})}{2 L}}
\quad\text{and}\quad
\left\|\widetilde{\mathbf{n}}-\mathbf{n}_{\tau}\right\|_2 \leq 2\|\mathbf{V}\|_2 \sqrt{2 B} \sqrt{\frac{\log (4 / \delta_\mathrm{L})}{L}},
\end{flalign}
and moreover $\widetilde{Z}\ge \frac{Z_{\tau,\min}}{2}$.
\end{lemma}

\begin{proof}
Let $\mathbf{V}\in\mathbb{R}^{N\times d}$ be the value matrix such that $\mathbf{V}^\top:= [\mathbf{v}_1, \dots \mathbf{v}_N] \in \mathbb{R}^{d \times N}$. For each table $\ell \in[L]$, define the score vector $\mathbf{s}^{(\ell)}(\qb):=(s_1^{(\ell)}(\qb), \ldots, s_N^{(\ell)}(\qb))^{\top} \in \mathbb{R}^N$ and $\|\mathbf{V}\|_2$ as the spectral norm of $\mathbf{V}$. From eq.~\eqref{eq:A2}, we already know $0\le Z^{(\ell)}(\qb)\le B$ for every $\ell\in [L]$. Next, we define the per-table numerator vector $\mathbf{n}^{(\ell)}(\qb):=\sum_{j=1}^N s_j^{(\ell)}(\qb) \vb_j=\mathbf{V}^\top \mathbf{s}^{(\ell)}(\qb)\in \mathbb{R}^d$. From the notations we already defined in the beginning of Section~\ref{sec:int}, we have
\begin{flalign}
&\widetilde{\mathbf{n}}(\qb)=\sum_{j=1}^N \widetilde{w}_j \vb_j=\frac{1}{L} \sum_{\ell=1}^L \sum_{j=1}^N s_j^{(\ell)}(\qb) \vb_j=\frac{1}{L} \sum_{\ell=1}^L \mathbf{n}^{(\ell)}(\qb)\in \mathbb{R}^d\\
&\widetilde{Z}=\sum_{j=1}^N \widetilde{w}_j=\sum_{j=1}^N\left(\frac{1}{L} \sum_{\ell=1}^L s_j^{(\ell)}(\qb)\right)=\frac{1}{L} \sum_{\ell=1}^L Z^{(\ell)}(\qb).
\end{flalign}

\paragraph{Bounding $\|\mathbf{n}^{(\ell)}(\qb)\|_2$:} Since $0\le s_j^{(\ell)}(\qb)\le 1$ for all $j\in [N]$, we have
\begin{flalign}
\|\mathbf{s}^{(\ell)}(\qb)\|_2^2=\sum_{j=1}^N (s_j^{(\ell)}(\qb))^2 \leq \sum_{j=1}^N s_j^{(\ell)}(\qb)=Z^{(\ell)}(\qb) \leq B.\label{eq:sq}
\end{flalign}
Now, combining eq.~\eqref{eq:sq} and submultiplicativity of the spectral norm, we have
\begin{flalign}
\|\mathbf{n}^{(\ell)}(\qb)\|_2=\|\mathbf{V}^\ts \mathbf{s}^{(\ell)}(\qb)\|_2 \leq\|\mathbf{V}\|_2\|\mathbf{s}^{(\ell)}(\qb)\|_2 \leq\|\mathbf{V}\|_2 \sqrt{B}.\label{eq:nq}
\end{flalign}

\paragraph{Bounding $\widetilde{Z}$:} From eq.~\eqref{eq:A2}, each $Z^{(\ell)}(\qb) \in[0, B]$, and the $Z^{(\ell)}$ are i.i.d. across $\ell$ with
\begin{flalign}
\mathbb{E}(\widetilde{Z})=&\mathbb{E}\left(\frac{1}{L} \sum_{\ell=1}^L Z^{(\ell)}(\qb)\right)= \mathbb{E}(Z^{(\ell)}(\qb))\nonumber\\
=&  \mathbb{E}\left(\sum_{j=1}^N s_j^{(\ell)}(\qb)\right)=\sum_{j=1}^N \mathbb{E}\left(s_j^{(\ell)}(\qb)\right)\nonumber\\
=& \sum_{j=1}^N w_{\tau, j} = Z_\tau.
\end{flalign}
Therefore, applying Hoeffding inequality with $t>0$, we get
$
\mathbb{P}\left(|\widetilde{Z}-Z_\tau | \geq t\right) \leq 2 \exp \left(-\frac{2 L t^2}{B^2}\right)$.
Taking $t=B \sqrt{\frac{\log (4 / \delta_\mathrm{L})}{2 L}}$, we have
\begin{flalign}
\mathbb{P}\left(|\widetilde{Z}-Z_\tau| \leq B \sqrt{\frac{\log (4 / \delta_\mathrm{L})}{2 L}}\right) \geq 1-\delta_\mathrm{L} / 2.\label{eq:hoef}
\end{flalign}
Now, if $B \sqrt{\frac{\log (4 / \delta_\mathrm{L})}{2 L}} \leq \frac{Z_{\tau, \min }}{2}$ or $L\ge\frac{2B^2\log (4/\delta_\mathrm{L})}{Z_{\tau, \min }^2}$ then combining Assumption~\ref{as1} with eq.~\eqref{eq:hoef} further boils down to
\begin{flalign}
\mathbb{P}\left(\widetilde{Z}\ge  \frac{Z_{\tau, \min }}{2}\right) \geq 1-\delta_\mathrm{L} / 2.\label{eq:st}
\end{flalign}

\paragraph{Bounding $\widetilde{\mathbf{n}}(\qb)$:} Let $\mathbf{X}_{\ell}:=\mathbf{n}^{(\ell)}(\qb)-\mathbf{n}_{\tau}(\qb) \in \mathbb{R}^d$. Then $\mathbb{E}\left[\mathbf{X}_{\ell}\right]=0$ and the $\mathbf{X}_{\ell}$ are i.i.d. Using eq.~\eqref{eq:nq} and Jensen's inequality, we have
\begin{flalign}
\|\mathbf{n}_{\tau}(\qb)\|_2=\|\mathbb{E}(\mathbf{n}^{(\ell)}(\qb))\|_2 \leq \mathbb{E}\|\mathbf{n}^{(\ell)}(\qb)\|_2 \leq\|\mathbf{V}\|_2 \sqrt{B}.\label{eq:ntq}
\end{flalign}
So, applying triangle inequality on the definition of $\mathbf{X}_{\ell}$ and combining eqs.~\eqref{eq:nq} \& \eqref{eq:ntq} yield
\begin{flalign}
\|\mathbf{X}_{\ell}\|_2 \leq\|\mathbf{n}^{(\ell)}(\qb)\|_2+\|\mathbf{n}_{\tau}(\qb)\|_2 \leq 2\|\mathbf{V}\|_2 \sqrt{B} .
\end{flalign}
Next, we apply a vector Hoeffding inequality for sums of bounded mean-zero random vectors in $\mathbb{R}^d$. Specifically, if $\mathbf{X}_1, \ldots, \mathbf{X}_K$ are independent with $\mathbb{E}\left[\mathbf{X}_{k}\right]=0$ and $\left\|\mathbf{X}_{k}\right\|_2 \leq R$, then
$$
\mathbb{P}\left(\left\|\sum_{k=1}^K \mathbf{X}_{k}\right\|_2 \geq t\right) \leq 2 \exp \left(-\frac{t^2}{2 R^2}\right),
$$
which follows from \emph{e.g.,} Theorem~3 in \cite{pinelis1992approach} and related vector concentration results (e.g., \cite{BLM13}). Applying this in our context directly yields
$$
\mathbb{P}\left(\left\|\widetilde{\mathbf{n}}-\mathbf{n}_{\tau}\right\|_2 \geq t\right) \leq 2 \exp \left(-\frac{L t^2}{8\|\mathbf{V}\|_2^2 B}\right).
$$
Setting $t=2\|\mathbf{V}\|_2 \sqrt{2 B} \sqrt{\frac{\log (4 / \delta_\mathrm{L})}{L}}$ above yields:
\begin{flalign}
\mathbb{P}\left(\left\|\widetilde{\mathbf{n}}-\mathbf{n}_{\tau}\right\|_2 \leq 2\|\mathbf{V}\|_2 \sqrt{2 B} \sqrt{\frac{\log (4 / \delta_\mathrm{L})}{L}}\right) \geq 1-\delta_\mathrm{L} / 2.\label{eq:hvec}
\end{flalign}
\noindent
Finally, taking the union bound over eqs.~\eqref{eq:hoef} and \eqref{eq:hvec} gives the stated joint event with probability at least $1-\delta_\mathrm{L}$, and under $L\ge\frac{2B^2\log (4/\delta_\mathrm{L})}{Z_{\tau, \min }^2}$ we also have $\widetilde{Z}\ge \frac{Z_{\tau,\min}}{2}$ by eq.~\eqref{eq:st}$.$
\end{proof}

\begin{lemma}[Final bound for $\yb_{\tau,L}(\qb)$]\label{lem:Llem}
If $L\ge\frac{2B^2\log (4/\delta_\mathrm{L})}{Z_{\tau, \min }^2}$, then
\[
\mathbb{P}\Big(\left\|\yb_{\tau, L}(\qb)-\yb_\tau(\qb)\right\|_2 \leq C\sqrt{\frac{\log (4 / \delta_\mathrm{L})}{L}}\|\mathbf{V}\|_2\Big)\ge 1-\delta_\mathrm{L},
\]
where $C=\left(\frac{4 \sqrt{2 B}}{Z_{\tau, \min }}+\frac{\sqrt{2} B^{3 / 2}}{Z_{\tau, \min }^2}\right)$.
\end{lemma}

\begin{proof}
Using the definitions of $\yb_{\tau, L}(\qb)$ and $\yb_\tau(\qb)$, and adding and subtracting $\mathbf{n}_{\tau} / \widetilde{Z}$, we have
\begin{flalign}
\yb_{\tau, L}(\qb)-\yb_\tau(\qb)=\frac{\widetilde{\mathbf{n}}}{\widetilde{Z}}-\frac{\mathbf{n}_{\tau}}{Z_\tau}=\frac{\widetilde{\mathbf{n}}-\mathbf{n}_{\tau}}{\widetilde{Z}}+\mathbf{n}_{\tau}\left(\frac{1}{\widetilde{Z}}-\frac{1}{Z_\tau}\right).
\end{flalign}
Taking norms both sides and applying triangle inequality yield
\begin{flalign}
\|\yb_{\tau, L}(\qb)-\yb_\tau(\qb)\|_2\le \frac{\|\widetilde{\mathbf{n}}-\mathbf{n}_{\tau}\|_2}{\widetilde{Z}} + \frac{\|\mathbf{n}_{\tau}\|_2}{\widetilde{Z} Z_\tau}|\widetilde{Z}-Z_\tau|.
\end{flalign}
Now using Assumption~\ref{as1} \emph{i.e.,} $Z_\tau \geq Z_{\tau, \min }$, $\widetilde{Z} \geq Z_{\tau, \min } / 2$ from Lemma~\ref{lem:zhat}, and $\left\|\mathbf{n}_{\tau}\right\|_2 \leq\|\mathbf{V}\|_2 \sqrt{B}$ in eq.~\eqref{eq:ntq}$:$ 
\begin{flalign}
\left\|\yb_{\tau, L}(\qb)-\yb_\tau(\qb)\right\|_2 \leq \frac{2}{Z_{\tau, \min }}\left\|\widetilde{\mathbf{n}}-\mathbf{n}_{\tau}\right\|_2+\frac{2\|\mathbf{V}\|_2 \sqrt{B}}{Z_{\tau, \min }^2}\left|\widetilde{Z}-Z_\tau\right|.
\end{flalign}
\noindent
Finally, applying the bounds
\[
\left\|\widetilde{\mathbf{n}}-\mathbf{n}_{\tau}\right\|_2 \leq 2\|\mathbf{V}\|_2 \sqrt{2 B} \sqrt{\frac{\log (4 / \delta_\mathrm{L})}{L}}
\quad\text{and}\quad
\left|\widetilde{Z}-Z_\tau\right| \leq B \sqrt{\frac{\log (4 / \delta_\mathrm{L})}{2 L}}
\]
(which hold jointly with probability at least $1-\delta_\mathrm{L}$ by Lemma~\ref{lem:zhat}), we get
\begin{flalign}
\left\|\yb_{\tau, L}(\qb)-\yb_\tau(\qb)\right\|_2 \leq& \frac{2}{Z_{\tau, \min }}\left(2\|\mathbf{V}\|_2 \sqrt{2 B} \sqrt{\frac{\log (4 / \delta_\mathrm{L})}{L}}\right)+\frac{2\|\mathbf{V}\|_2 \sqrt{B}}{Z_{\tau, \min }^2}\left(B \sqrt{\frac{\log (4 / \delta_\mathrm{L})}{2 L}}\right)\nonumber\\
=& \left(\frac{4 \sqrt{2 B}}{Z_{\tau, \min }}+\frac{\sqrt{2} B^{3 / 2}}{Z_{\tau, \min }^2}\right) \sqrt{\frac{\log (4 / \delta_\mathrm{L})}{L}}\|\mathbf{V}\|_2.
\end{flalign}
This concludes the proof.
\end{proof}
%%%%%%%%%%%%%%%%%%%%%%%%%%%%%%%%%%%%%%%%%%%%%%%%%%
\noindent
Conditioned on the tables, define $S_1:=\sum_{j=1}^N \widetilde a_j\|\vb_j\|_2$ and sampling
probabilities $p_j:=\widetilde a_j\|\vb_j\|_2/S_1$. Draw i.i.d. indices
$J_1,\dots,J_M\sim p$ and define
\[
\mathbf{T}(\qb):=\frac{1}{M}\sum_{m=1}^M \frac{\widetilde a_{J_m}}{p_{J_m}}\,\mathbf{v}_{J_m}\in\mathbb{R}^d.
\]
\noindent
Let $\Wcal=\{\mathbf{W}_1,\dots,\mathbf{W}_L\}$ be the set of random Gaussian hyperplanes with $\mathbf{W}_\ell\in\RR{P}{d}$. Now are ready to present the results for sampling.

\begin{lemma}\label{lem:sampling}
The following properties hold:
\begin{enumerate}
    \item $\mathbb{E}[\mathbf{T}(\qb) \mid \mathcal{W}]=\yb_{\tau, L}(\qb)$
    \item $\PP\left(\left\|\mathbf{T}(\qb)-\yb_{\tau, L}(\qb)\right\|_2 \leq \|\mathbf{V}\|_2 \sqrt{\frac{8 \log \left(2 / \delta_\mathrm{M}\right)}{M}}\right)\ge 1-\delta_\mathrm{M}$
\end{enumerate}
\end{lemma}

\begin{proof}

\textbf{Proof of Part 1.}
\noindent
Let $J_1,\dots,J_M$ be the $M$ independent samples of $\{1,\dots,N\}$ drawn using the probability distribution $p:=(p_1,\dots,p_N)$. Define $\mathbf{X}_m:=\frac{\widetilde{a}_{J_m}}{p_{J_m}} \mathbf{v}_{J_m} \in \mathbb{R}^d$. Then
$$
\mathbf{T}(\qb)=\frac{1}{M} \sum_{m=1}^M \mathbf{X}_m
$$
Now, given $\mathcal{W}$. Then $\widetilde{a}_j$ and $p_j$ are fixed numbers, so:

$$
\mathbb{E}[\mathbf{X}_m \mid \mathcal{W}]=\sum_{j=1}^N p_j \frac{\widetilde{a}_j}{p_j} \vb_j=\sum_{j=1}^N \widetilde{a}_j \vb_j=\yb_{\tau, L}(\qb) .
$$
\noindent
Therefore,
$$
\mathbb{E}[\mathbf{T}(\qb) \mid \mathcal{W}]=\frac{1}{M} \sum_{m=1}^M \mathbb{E}\left[\mathbf{X}_m \mid \mathcal{W}\right]=\yb_{\tau, L}(\qb)
$$

\noindent
\textbf{Proof of Part 2.}
\noindent
We define the centered random vector
$\Yb_m:=\mathbf{X}_m-\mathbb{E}\left[\mathbf{X}_m \mid \mathcal{W}\right]=\mathbf{X}_m-\yb_{\tau, L}(\qb)$. Then $\mathbb{E}\left[\Yb_m \mid \mathcal{W}\right]=0$ and $\mathbf{T}(\qb)-\yb_{\tau, L}(\qb)=\frac{1}{M} \sum_{m=1}^M \mathbf{Y}_m$. Now, using $p_{J_m}:=\frac{\widetilde a_{J_m}\|\mathbf{v}_{J_m}\|_2}{S_1}$, we have

\begin{flalign}
\|\mathbf{X}_m\|_2=\left\|\frac{\widetilde{a}_{J_m}}{p_{J_m}} \mathbf{v}_{J_m}\right\|_2= \frac{\widetilde{a}_{J_m}}{\nicefrac{\widetilde a_{J_m}\|\mathbf{v}_{J_m}\|_2}{S_1}}\|\mathbf{v}_{J_m}\|_2=S_1
\end{flalign}
\noindent
and also applying triangle inequality

\begin{flalign}
\|\yb_{\tau, L}(\qb)\|_2=\Big\|\sum_{j=1}^N \widetilde{a}_j \vb_j\Big\|_2\le \sum_{j=1}^N \widetilde{a}_j\|\vb_j\|_2=S_1.
\end{flalign}
\noindent
Therefore, applying the triangle inequality on $\mathbf{Y}_m$, we have
\begin{flalign}
\|\mathbf{Y}_m\|_2\le \|\mathbf{X}_m\|_2+\|\yb_{\tau, L}(\qb)\|_2=2S_1\label{eq:2s1}
\end{flalign}
\noindent
Now, note that we can always rewrite $\vb_j=\mathbf{V}^\ts \mathbf{e}_j$ where $\mathbf{e}_j\in\R{N}$ is the $j$-th standard basis vector and thus $\|\mathbf{e}_j\|_2=1$. Consequently, applying the submultiplicativity property, we trivially have 

\begin{flalign}
    \|\vb_j\|_2=\|\mathbf{V}^\ts \mathbf{e}_j\|_2 \le \|\mathbf{e}_j\|_2\|\mathbf{V}\|_2=\|\mathbf{V}\|_2\label{eq:vnorm}
\end{flalign}
\noindent
Using the definition of $S_1$, eq.~\eqref{eq:vnorm} and the fact that $\sum_{j=1}^N \widetilde{a}_j$=1, we get
\begin{flalign}
S_1= \sum_{j=1}^N \widetilde{a}_j\|\vb_j\|_2\le \|\mathbf{V}\|_2
\end{flalign}
\noindent
Therefore $\|\mathbf{Y}_m\|_2\le 2\|\mathbf{V}\|_2$. Now, using the fact that $\mathbf{T}(\qb)-\yb_{\tau, L}(\qb)=\frac{1}{M} \sum_{m=1}^M \mathbf{Y}_m
$ and applying the same vector Hoeffding bound we have

\begin{flalign}
\PP\left(\left\|\mathbf{T}(\qb)-\yb_{\tau, L}(\qb)\right\|_2 \geq t\mid\Wcal\right)=\PP\left(\Big\|\frac{1}{M} \sum_{m=1}^M \mathbf{Y}_m\Big\|_2 \geq t\mid\Wcal\right)\le 2 \exp \left(-\frac{M t^2}{8 \|\mathbf{V}\|_2^2}\right)\label{eq:hoeffvec}
\end{flalign}
\noindent
Now, for a failure probability $0<\delta_\mathrm{M}<1$, we take $t=\sqrt{\frac{8\log (2/\delta_\mathrm{M})}{M}}\|\mathbf{V}\|_2$ and eq.~\eqref{eq:hoeffvec} further simplify to 
\begin{flalign}
 \PP\left(\left\|\mathbf{T}(\qb)-\yb_{\tau, L}(\qb)\right\|_2 \leq \sqrt{\frac{8\log (2/\delta_\mathrm{M})}{M}}\|\mathbf{V}\|_2\mid\Wcal\right)\ge 1-\delta_\mathrm{M}\label{eq:cond} 
\end{flalign}
\noindent
Finally, since eq.~\eqref{eq:cond} holds for any $\Wcal$, we take the expectations over both sides with respect to $\Wcal$ and use the law of iterated expectation to get

\begin{flalign}
\PP\Big(\big\|\mathbf{T}(\qb)-\yb_{\tau, L}(\qb)\big\|_2 \leq \sqrt{\nicefrac{8\log (2/\delta_\mathrm{M})}{M}}\|\mathbf{V}\|_2\Big)\ge 1-\delta_\mathrm{M}    
\end{flalign}

\end{proof}

\begin{lemma}\label{lem:disc}
For any fixed query $\qb$, we have
\[
\left\|\yb_\tau(\qb)-\yb^*(\qb)\right\|_2
\leq
2 B\left(\frac{1}{Z_{\tau, \mathrm{min}}}+\frac{\sqrt{B}}{Z_{\mathrm{min}} Z_{\tau, \mathrm{min}}}\right)
\varepsilon_\tau(\qb)\,\|\mathbf{V}\|_2 .
\]
\end{lemma}

\begin{proof}
From the definition of $\yb_\tau(\qb)$ and $\yb^*(\qb)$ we have
\begin{flalign}
\|\yb_\tau(\qb)-\yb^*(\qb)\|_2=\big\|\sum_{j=1}^N\left(a_{\tau, j}-a_j\big) \vb_j\right\|_2= \left\|\mathbf{V}^\ts\left(\ab_{\tau}-\ab\right)\right\|_2\le\|\mathbf{V}\|_2\|\ab_{\tau}-\ab\|_2\,,\label{eq:mis}
\end{flalign}
where $\ab_{\tau}$ and $\ab$ are the vectors of order $N$ whose $j$-th entries are $a_{\tau, j}$ and $a_j$ respectively. We applied the submultiplicativity property of vector and matrix norms in the last inequality of eq.~\eqref{eq:mis}. So it remains to bound $\left\|\ab_\tau-\ab\right\|_2$. Now, from the definitions of $a_{\tau, j}$ and $a_j$ in the beginning of Section~\ref{sec:int}, we rewrite $(\ab_\tau-\ab)$ as
\begin{flalign}
\ab_\tau-\ab=\frac{\wb_\tau}{Z_\tau}-\frac{\wb}{Z}
=\frac{\wb_\tau-\wb}{Z_\tau}+\frac{Z-Z_\tau}{Z_\tau Z}\,\wb\,,\label{eq:attker}
\end{flalign}
where $\wb_\tau$ and $\wb$ are the unnormalized attention vectors of dimension $N$ whose $j$-th entries are $w_{\tau,j}$ and $w_j$ respectively. The last equality follows from adding and subtracting $\nicefrac{\wb} {Z_\tau}$ from the second equality and then simplifying. Now, taking $\ell_2$-norm on the both sides of eq.~\eqref{eq:attker} and applying triangle inequality, we further have
\begin{flalign}
\|\ab_{\tau}-\ab\|_2\le&\frac{\|\wb_\tau-\wb\|_2}{Z_\tau}+\frac{|Z-Z_\tau|}{Z_\tau Z}\,\|\wb\|_2\nonumber\\
\le& \frac{\|\wb_\tau-\wb\|_2}{Z_{\tau ,\text{min}}}+\frac{|Z-Z_\tau|}{Z_{\tau ,\text{min}}\, Z_{\text{min}}}\,\|\wb\|_2\,,\label{eq:abbo}
\end{flalign}
where the last inequality in eq.~\eqref{eq:abbo} follows directly from Assumption~\ref{as1}.

\noindent
Now, we rewrite $w_j$ as an SRP collision probability.
Fix a single table with $\Wb^{(\ell)}\sim\mathcal{N}(0,1)^{P\times d}$.
Define the hard hash (bucket id) of any $\xb\in\R{d}$ by
\[
h(\xb):=\mathrm{sign}(\Wb^{(\ell)}\xb)\in\{\pm 1\}^P
\quad\text{(equivalently an index in $[R]$)}.
\]
Then $b_j^{(\ell)}$ is precisely the index of $h(\kb_j)$.
Let $b_q:=h(\qb)$ denote the query bucket under the same table.
Define the hard collision indicator
$I_j:=\mathbf{1}\{b_j^{(\ell)}=b_q\}\in\{0,1\}.$
By the standard SRP/SimHash collision identity, we rewrite
\[
w_j
=\Big(1-\frac{1}{\pi}\cos^{-1}\Big(\frac{\qb^\top \kb_j}{\|\qb\|\,\|\kb_j\|}\Big)\Big)^P
=
\mathbb{E}[I_j],
\]
where the expectation is over $\Wb^{(\ell)}$. We now express $w_{\tau,j}$ as the soft score expectation and define the peaking quantity.
Recall $w_{\tau,j}:=\EE[s_j^{(\ell)}(\qb)]$ where
$s_j^{(\ell)}(\qb)=p_\tau^{(\ell)}(b_j^{(\ell)}\mid \qb)\in[0,1].$
For the same table, define the (random) dominant query bucket
\[
b_\star:=\arg\max_{r\in[R]} p_\tau^{(\ell)}(r\mid \qb).
\]
Since $p_\tau^{(\ell)}(r\mid \qb)\propto \exp(\ub^{(\ell)}(\qb)^\top \mathbf{c}_r/\tau)$,
$\ub^{(\ell)}(\qb)=\frac{1}{\sqrt d}\tanh(\Wb^{(\ell)}\qb)$, and $\tanh (\cdot)$ is strictly increasing, we have $\operatorname{sign}\left(\ub^{(\ell)}(\qb)\right)=\operatorname{sign}\left(\Wb^{(\ell)} \qb\right)$  coordinatewise.
Hence the dominant query bucket under soft collision and the hard bucket for $\qb$ coincides \ie, $b_\star=b_q$.
Now, we define the non-dominant mass quantity
\[
\varepsilon_\tau(\qb)
:=
\EE\big[1-p_\tau^{(\ell)}(b_q\mid \qb)\big]
=
\EE\big[1-p_\tau^{(\ell)}(b_\star\mid \qb)\big]\,,
\]

\noindent
where the expectation is over the randomness of the table $\Wb^{(\ell)}$. In fact, $\varepsilon_\tau(\qb)$ quantifies how much probability mass the soft bucket distribution assigns to buckets other than the hard SRP bucket,
and it controls the discrepancy between soft-count weights and hard collision probabilities. 
\noindent
Now, we are ready to bound $\|\wb_\tau-\wb\|_1$ and $|Z-Z_\tau|$ using occupancy $B$.
We first prove a \emph{per-table} inequality. For a fixed realization of $\Wb^{(\ell)}$,
we group indices by buckets and use Assumption~\ref{as2}.
Write
\begin{flalign}
\sum_{j=1}^N |s_j^{(\ell)}(\qb)-I_j|
=
\sum_{b_j^{(\ell)}=b_q}\!\!\big|p_\tau^{(\ell)}(b_q\mid\qb)-1\big|
+
\sum_{b_j^{(\ell)}\neq b_q}\!\!p_\tau^{(\ell)}(b_j^{(\ell)}\mid\qb).\label{eq:bound}
\end{flalign}
For the first term, there are at most $B$ indices in bucket $b_q$, hence
\begin{flalign}
\sum_{b_j^{(\ell)}=b_q}\!\!\big|p_\tau^{(\ell)}(b_q\mid\qb)-1\big|
=
\Big(\sum_{b_j^{(\ell)}=b_q}1\Big)\big(1-p_\tau^{(\ell)}(b_q\mid\qb)\big)
\le B\big(1-p_\tau^{(\ell)}(b_q\mid\qb)\big).\label{eq:bound1}
\end{flalign}
For the second term, we regroup by buckets:
\begin{flalign}
\sum_{b_j^{(\ell)}\neq b_q} p_\tau^{(\ell)}(b_j^{(\ell)}\mid\qb)
=
\sum_{r\neq b_q}\ \sum_{b_j^{(\ell)}=r} p_\tau^{(\ell)}(r\mid\qb)
\le
\sum_{r\neq b_q} B\,p_\tau^{(\ell)}(r\mid\qb)
=
B\big(1-p_\tau^{(\ell)}(b_q\mid\qb)\big).\label{eq:bound2}
\end{flalign}
In the above, the first equality in eq.~\eqref{eq:bound2} utilizes the fact that each bucket $r \in[R]$, all keys $j$ with $b_j^{(\ell)}=r$ contribute the same probability value $p_\tau^{(\ell)}(r \mid \qb)$ (because the probability depends only on the bucket $r$, not on $j$ ). Then, inside the inner sum, the term $p_\tau^{(\ell)}(r \mid \qb)$ is constant w.r.t. $j$ and the inequality follows directly from Assumption~\ref{as2}.

\noindent
Combining eqs.~\eqref{eq:bound}, \eqref{eq:bound1} and \eqref{eq:bound2} yield the deterministic bound
\begin{flalign}
\sum_{j=1}^N |s_j^{(\ell)}(\qb)-I_j|
\le
2B\big(1-p_\tau^{(\ell)}(b_q\mid\qb)\big).\label{eq:soft-hard-pertable}
\end{flalign}
\noindent
Now take expectation over $\Wb^{(\ell)}$ and use $w_{\tau,j}=\EE[s_j^{(\ell)}(\qb)]$ and $w_j=\EE[I_j]$:
\begin{flalign}
\|\wb_\tau-\wb\|_1
&=\sum_{j=1}^N |w_{\tau,j}-w_j|
=\sum_{j=1}^N \big|\EE[s_j^{(\ell)}(\qb)-I_j]\big|
\le \sum_{j=1}^N \EE|s_j^{(\ell)}(\qb)-I_j| \nonumber\\
&=\EE\Big[\sum_{j=1}^N |s_j^{(\ell)}(\qb)-I_j|\Big]
\le 2B\,\EE\big[1-p_\tau^{(\ell)}(b_q\mid\qb)\big]
=2B\,\varepsilon_\tau(\qb).\label{eq:w1}
\end{flalign}
\noindent
Since $\|\cdot\|_2\le \|\cdot\|_1$, we immediately get
\begin{flalign}
\|\wb_\tau-\wb\|_2 \le \|\wb_\tau-\wb\|_1 \le 2B\,\varepsilon_\tau(\qb).\label{eq:w2}
\end{flalign}
\noindent
Next, note that $Z=\sum_{j=1}^N w_j$ and $Z_\tau=\sum_{j=1}^N w_{\tau,j}$, hence
\begin{flalign}
|Z-Z_\tau|
=
\Big|\sum_{j=1}^N (w_j-w_{\tau,j})\Big|
\le
\sum_{j=1}^N |w_j-w_{\tau,j}|
=
\|\wb_\tau-\wb\|_1
\le 2B\,\varepsilon_\tau(\qb).\label{eq:Zdiff}
\end{flalign}

\noindent
Next, to bound $\|\wb\|_2$ using occupancy $B$,
we use that $w_j\in[0,1]$ and the following shows $Z\le B$.
Indeed, for a fixed table,
\[
\sum_{j=1}^N I_j
=
\#\{j\in[N]: b_j^{(\ell)}=b_q\}
\le B
\quad\text{by Assumption~\ref{as2}}.
\]
Taking expectation over $\Wb^{(\ell)}$ gives
\[
Z=\sum_{j=1}^N w_j
=\sum_{j=1}^N \EE[I_j]
=
\EE\Big[\sum_{j=1}^N I_j\Big]
\le B.
\]
Moreover, since $0\le w_j\le 1$,
\[
\|\wb\|_2^2=\sum_{j=1}^N w_j^2 \le \sum_{j=1}^N w_j = Z \le B,
\]
hence
\begin{flalign}
\|\wb\|_2 \le \sqrt{B}.\label{eq:w2bound}
\end{flalign}

\noindent
Next,
we substitute eqs.~\eqref{eq:w2}, \eqref{eq:Zdiff}, and \eqref{eq:w2bound} into eq.~\eqref{eq:abbo}:
\begin{flalign}
\|\ab_{\tau}-\ab\|_2
&\le
\frac{\|\wb_\tau-\wb\|_2}{Z_{\tau,\mathrm{min}}}
+
\frac{|Z-Z_\tau|}{Z_{\tau,\min}Z_{\min}}\,\|\wb\|_2 \nonumber\\
&\le
\frac{2B\,\varepsilon_\tau(\qb)}{Z_{\tau,\min}}
+
\frac{2B\,\varepsilon_\tau(\qb)}{Z_{\tau,\min}Z_{\min}}\cdot \sqrt{B} \nonumber\\
&=
2B\Big(\frac{1}{Z_{\tau,\min}}+\frac{\sqrt{B}}{Z_{\min}Z_{\tau,\min}}\Big)\varepsilon_\tau(\qb).\label{eq:ab-final}
\end{flalign}

\noindent
Finally, combining eq.~\eqref{eq:mis} with eq.~\eqref{eq:ab-final}:
% and eq.~\eqref{eq:peaking}:
\begin{flalign}
\|\yb_\tau(\qb)-\yb^*(\qb)\|_2
% &\le \|\mathbf{V}\|_2\,\|\ab_\tau-\ab\|_2 \nonumber\\
&\le
\|\mathbf{V}\|_2\cdot
2B\Big(\frac{1}{Z_{\tau,\min}}+\frac{\sqrt{B}}{Z_{\min}Z_{\tau,\min}}\Big)\varepsilon_\tau(\qb).
\end{flalign}
This proves the lemma.

\end{proof}

\subsection{Proof of Theorem~\ref{thm:main}}\label{proof_final}
\begin{theorem}[Final end-to-end bound via union bound]\label{thm:final}
Fix a query $\qb$. Suppose Assumptions~\ref{as1}--\ref{as2} hold and let
$\delta_\mathrm{L},\delta_\mathrm{M}\in(0,1)$. Assume
\[
L \;\ge\; \frac{2B^2\log(4/\delta_\mathrm{L})}{Z_{\tau,\min}^2}.
\]
Then, with probability at least $1-(\delta_\mathrm{L}+\delta_\mathrm{M})$,
\begin{align}
\|\mathbf{T}(\qb)-\yb^*(\qb)\|_2
\;\le\;&
\underbrace{\|\mathbf{V}\|_2\sqrt{\frac{8\log(2/\delta_\mathrm{M})}{M}}}_{\text{sampling error}}
\;+\;
\underbrace{C\,\|\mathbf{V}\|_2\sqrt{\frac{\log(4/\delta_\mathrm{L})}{L}}}_{\text{finite-$L$ error}}
\nonumber\\
&\;+\;
\underbrace{2B\Big(\frac{1}{Z_{\tau,\min}}+\frac{\sqrt{B}}{Z_{\min}Z_{\tau,\min}}\Big)
\,\varepsilon_\tau(\qb)\,\|\mathbf{V}\|_2}_{\text{soft vs.\ angular bias}},
\end{align}
where
\[
C=\Big(\frac{4\sqrt{2B}}{Z_{\tau,\min}}+\frac{\sqrt{2}B^{3/2}}{Z_{\tau,\min}^2}\Big),
\qquad
\varepsilon_\tau(\qb)=\mathbb{E}\!\left[1-p_\tau^{(\ell)}(b_q\mid \qb)\right].
\]
\end{theorem}

\begin{proof}
By the triangle inequality,
\begin{align}
\|\mathbf{T}(\qb)-\yb^*(\qb)\|_2
&\le
\|\mathbf{T}(\qb)-\yb_{\tau,L}(\qb)\|_2
+\|\yb_{\tau,L}(\qb)-\yb_\tau(\qb)\|_2
+\|\yb_\tau(\qb)-\yb^*(\qb)\|_2 .
\label{eq:triangle}
\end{align}
\noindent
Define the events
\begin{align*}
\mathcal{E}_\mathrm{M}
&:=\Big\{\|\mathbf{T}(\qb)-\yb_{\tau,L}(\qb)\|_2
\le \|\mathbf{V}\|_2\sqrt{\tfrac{8\log(2/\delta_\mathrm{M})}{M}}\Big\},\\
\mathcal{E}_\mathrm{L}
&:=\Big\{\|\yb_{\tau,L}(\qb)-\yb_\tau(\qb)\|_2
\le C\,\|\mathbf{V}\|_2\sqrt{\tfrac{\log(4/\delta_\mathrm{L})}{L}}\Big\}.
\end{align*}
By the Lemma~\ref{lem:sampling},
$\mathbb{P}(\mathcal{E}_\mathrm{M})\ge 1-\delta_\mathrm{M}$,
and by Lemma~\ref{lem:Llem},
$\mathbb{P}(\mathcal{E}_\mathrm{L})\ge 1-\delta_\mathrm{L}$.
\noindent
Moreover, Lemma~\ref{lem:disc} gives the deterministic bound
\[
\|\yb_\tau(\qb)-\yb^*(\qb)\|_2
\le
2B\Big(\frac{1}{Z_{\tau,\min}}+\frac{\sqrt{B}}{Z_{\min}Z_{\tau,\min}}\Big)
\varepsilon_\tau(\qb)\,\|\mathbf{V}\|_2 .
\]
\noindent
On the intersection $\mathcal{E}_\mathrm{M}\cap\mathcal{E}_\mathrm{L}$,
substituting these three bounds into~\eqref{eq:triangle} yields the claimed
inequality. Finally, by the union bound,
\begin{flalign}
\mathbb{P}(\mathcal{E}_\mathrm{M}\cap\mathcal{E}_\mathrm{L})
\ge 1-\delta_\mathrm{M}-\delta_\mathrm{L}.\label{eq:union}
\end{flalign}
\noindent
We conclude the proof by taking $\delta_M=\delta_L=\frac{\delta}{2}$ in eq.~\eqref{eq:union}.
\end{proof}

\section{Additional Proofs}
\label{app:add_proofs}

\begin{proof}[Proof of Lemma~\ref{lem:orth-corr}]
We first note that $X=\qb^\top\kb\sim\mathcal N(0,\|\qb\|_2^2)=\mathcal N(0,1)$, hence $\EE[X]=0$ and $\EE[X^2]=1$.
Also, for each $i$, $\hat{\wb}_i^\top\kb\sim\mathcal N(0,1)$ is symmetric about $0$, so $\EE[\operatorname{sign}(\hat{\wb}_i^\top\kb)]=0$ and therefore $\EE[Y]=\sum_i s_i\,\EE[\operatorname{sign}(\hat{\wb}_i^\top\kb)]=0$.
\noindent
For $\EE[Y^2]$, expand
$Y^2=\sum_{i=1}^P s_i^2 + \sum_{i\neq j} s_is_j\,\operatorname{sign}(\hat{\wb}_i^\top\kb)\operatorname{sign}(\hat{\wb}_j^\top\kb)$,
so
$\EE[Y^2]=\sum_{i=1}^P s_i^2+\sum_{i\neq j}s_is_j\,\EE[\operatorname{sign}(\hat{\wb}_i^\top\kb)\operatorname{sign}(\hat{\wb}_j^\top\kb)]$.
For $i\neq j$, the pair $(\hat{\wb}_i^\top\kb,\hat{\wb}_j^\top\kb)$ is jointly Gaussian with correlation
$\rho_{ij}=\hat{\wb}_i^\top\hat{\wb}_j$.
Using the standard identity $\EE[\operatorname{sign}(U)\operatorname{sign}(V)]=\frac{2}{\pi}\arcsin(\rho_{ij})$ for jointly Gaussian $(U,V)$, and the orthogonality assumption $\rho_{ij}=0$, we get
$\EE[\operatorname{sign}(\hat{\wb}_i^\top\kb)\operatorname{sign}(\hat{\wb}_j^\top\kb)]=0$; thus all cross terms vanish and $\EE[Y^2]=\sum_{i=1}^P s_i^2$.
\\
\\
\noindent
For $\EE[XY]$, by linearity,
$\EE[XY]=\sum_{i=1}^P s_i\,\EE[(\qb^\top\kb)\operatorname{sign}(\hat{\wb}_i^\top\kb)]$.
Fix $i$ and decompose $\kb=r\hat{\wb}_i+\ub$ where $r:=\hat{\wb}_i^\top\kb\sim\mathcal N(0,1)$ and $\ub\perp \hat{\wb}_i$ is independent of $r$.
Then $\operatorname{sign}(\hat{\wb}_i^\top\kb)=\operatorname{sign}(r)$ and $\qb^\top\kb=r(\qb^\top\hat{\wb}_i)+\qb^\top\ub$.
By independence and symmetry, $\EE[\qb^\top\ub\cdot\operatorname{sign}(r)]=\EE[\qb^\top\ub]\EE[\operatorname{sign}(r)]=0$.
Also $r\operatorname{sign}(r)=|r|$, so $\EE[(\qb^\top\kb)\operatorname{sign}(\hat{\wb}_i^\top\kb)]=(\qb^\top\hat{\wb}_i)\EE|r|$.
Summing over $i$ yields $\EE[XY]=C\sum_{i=1}^P (\qb^\top\hat{\wb}_i)s_i$ with $C=\EE|r|=\sqrt{2/\pi}$.

Finally, since $\EE[X]=\EE[Y]=0$, the correlation is
$\Gamma=\frac{\EE[XY]}{\sqrt{\EE[X^2]\EE[Y^2]}}
= C\,\frac{\sum_i (\qb^\top\hat{\wb}_i)s_i}{\sqrt{\sum_i s_i^2}}$.
Writing $\sum_i (\qb^\top\hat{\wb}_i)s_i=\qb^\top \Wb^\ts\sbb$ and $\|\sbb\|_2=\sqrt{\sum_i s_i^2}$ gives
$\Gamma=C\,\qb^\top\Wb^\ts\hat{\sbb}$.
\end{proof}

\textbf{Hard vs.\ soft scoring:} Lemma~\ref{lem:orth-corr} shows that, under approximately orthogonal random planes,
the correlation between the true similarity $X=\qb^\top\kb$ and the aggregated score $Y$
takes the form
\[
\Gamma
=
C\,\qb^\top \Wb^\top \hat{\sbb}
=
C\,\frac{(\Wb\qb)^\top \sbb}{\|\sbb\|_2},
\qquad
C=\sqrt{\frac{2}{\pi}},
\]
where $\sbb=(s_1,\ldots,s_P)^\top$ collects the per-plane scores and
$\hat{\sbb}=\sbb/\|\sbb\|_2$.
Thus, the behavior of the scoring rule is governed by the alignment between the projected
query $\Wb\qb$ and the score vector $\sbb$.
\\
\\
\noindent
\emph{Hard scoring:} For hard LSH, $s_i^{\mathrm{hard}}=\operatorname{sign}(\qb^\top\hat{\wb}_i)$, equivalently
$\sbb^{\mathrm{hard}}=\operatorname{sign}(\Wb\qb)$.
In this case $\|\sbb^{\mathrm{hard}}\|_2=\sqrt{P}$ and
\[
\Gamma_{\mathrm{hard}}
=
C\,\frac{(\Wb\qb)^\top \operatorname{sign}(\Wb\qb)}{\sqrt{P}}
=
\frac{C}{\sqrt{P}}\|\Wb\qb\|_1.
\]
Hence, hard scoring depends only on the $\ell_1$ magnitude of the projected coordinates and
discards all directional information within each orthant. As a result, small perturbations
in $\Wb\qb$ can induce discontinuous changes in $\hat{\sbb}^{\mathrm{hard}}$, leading to
slow concentration and unstable rankings at finite $L$.
\\
\\
\noindent
\emph{Soft scoring:}
For soft scoring, $s_i^{\mathrm{soft}}=\tanh(\qb^\top\hat{\wb}_i)$, i.e.,
$\sbb^{\mathrm{soft}}=\tanh(\Wb\qb)$ (applied elementwise). The corresponding correlation is
\[
\Gamma_{\mathrm{soft}}
=
C\,\frac{(\Wb\qb)^\top \tanh(\Wb\qb)}{\|\tanh(\Wb\qb)\|_2}.
\]
In the small-signal regime typical of high-dimensional random projections,
$|(\Wb\qb)_i|\ll 1$ and $\tanh(\Wb\qb)\approx \Wb\qb$, yielding
\[
\Gamma_{\mathrm{soft}}
\;\approx\;
C\,\|\Wb\qb\|_2,
\qquad
\hat{\sbb}^{\mathrm{soft}}
\;\approx\;
\frac{\Wb\qb}{\|\Wb\qb\|_2}.
\]
Thus, soft scoring preserves the directional structure of the projected coordinates and
varies smoothly with $\qb$.
\\
\\
\noindent
\emph{Comparison:}
By the inequality $\|\xb\|_1\le\sqrt{P}\|\xb\|_2$, we have
\[
\Gamma_{\mathrm{hard}}
=
\frac{C}{\sqrt{P}}\|\Wb\qb\|_1
\;\le\;
C\,\|\Wb\qb\|_2
\;\approx\;
\Gamma_{\mathrm{soft}},
\]
with strict inequality unless the coordinates of $\Wb\qb$ have equal magnitude.
Consequently, soft scoring achieves strictly stronger alignment with the true similarity
signal in the regime relevant for ranking-based inference, explaining its faster
concentration and improved top-$k$ stability at finite $L$.

\section{Pseudocode for the custom scoring CUDA kernel}

\begin{algorithm}[H]
\caption{SoftHashCollisionScores (query-side)}
\label{alg:soft-hash-collision}
\begin{algorithmic}[1]
\State {\bfseries Input:} Query $\qb$; BucketProbs
$p_{\tau}^{(\ell)}(r \mid \qb)$ for all tables $\ell \in [L]$ and buckets
$r \in [R]$ (from Algorithm~\ref{alg:soft-bucket-probs}); bucket ids
$b_j^{(\ell)} \in [R]$ for all keys $j \in [N]$ and tables $\ell \in [L]$
(from Algorithm~\ref{alg:race-precompute}); value weights
$\nu_j \ge 0$ for all $j \in [N]$; mask $m_j \in \{0,1\}$ indicating whether
key $j$ is valid.
\vspace{0.25em}

\State {\bfseries Per-key computation (one thread per $j$):}
\For{$j = 1$ to $N$}
  \If{$m_j = 0$}
    \State $\widehat w_j(\qb) \gets -\infty$
  \Else
    \State $\displaystyle
      \widehat w_j(\qb) \gets
      ||v||_j \cdot \sum_{\ell=1}^{L}
        p_{\tau}^{(\ell)}\big(b_j^{(\ell)} \mid \qb\big)$
  \EndIf
\EndFor
\vspace{0.25em}

\State {\bfseries Output:} Collision scores
$\{\widehat w_j(\qb)\}_{j=1}^N$ for the $N$ keys.
\end{algorithmic}
\end{algorithm}

\begin{figure}
    \centering
    \includegraphics[width=1\linewidth]{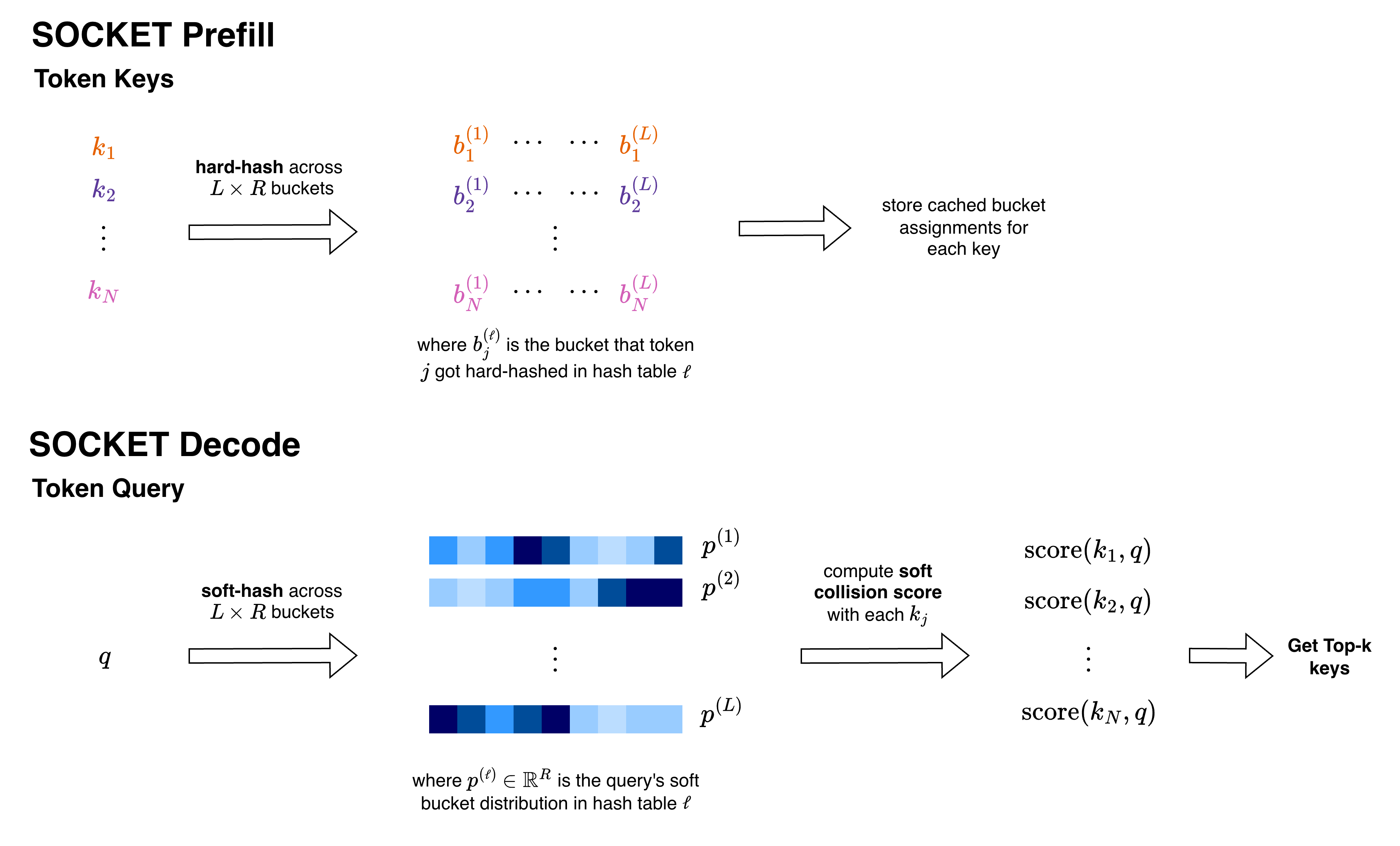}
    \caption{Overview of the SOCKET design and workflow.}
    \label{fig:Socket_workflow}
\end{figure}

\end{document}